\documentclass{article}

\usepackage[main, final]{neurips_2026}

\usepackage[utf8]{inputenc} 
\usepackage[T1]{fontenc}    
\usepackage[colorlinks=true,linkcolor=blue,citecolor=blue,urlcolor=blue]{hyperref}      
\usepackage{xurl}            
\usepackage{booktabs}       
\usepackage{amsfonts}       
\usepackage{nicefrac}       
\usepackage{microtype}      
\usepackage{xcolor}         

\usepackage{amsmath,amssymb,mathtools}
\usepackage{amsthm}
\usepackage{physics}
\mathtoolsset{showonlyrefs}  

\usepackage{graphicx}
\usepackage{subcaption}   
\usepackage{float}
\usepackage{wrapfig}
\usepackage{tikz}
\usepackage{multirow}

\usepackage{csquotes}
\usepackage{etoolbox}
\usepackage{makecell}
\usepackage{enumitem}     
\usepackage{pifont}
\usepackage{soul}         
\usepackage{cancel}
\usepackage{bbm}
\usepackage{mathrsfs}
\newcommand{\compactpara}[1]{\noindent\textbf{#1}\quad}

\usepackage[ruled]{algorithm2e}
\DontPrintSemicolon

\theoremstyle{plain}
\newtheorem{theorem}{Theorem}[section]
\newtheorem{lemma}[theorem]{Lemma}

\theoremstyle{definition}
\newtheorem{definition}[theorem]{Definition}
\newtheorem{assumption}[theorem]{Assumption}
\theoremstyle{remark}

\newcommand{\fY}{\mathcal{Y}}\newcommand{\fR}{\mathcal{R}}\newcommand{\fB}{\mathcal{B}}
\newcommand{\fZ}{\mathcal{Z}}\newcommand{\fT}{\mathcal{T}}
\newcommand{\fG}{\mathcal{G}}\newcommand{\fD}{\mathcal{D}}
\newcommand{\fX}{\mathcal{X}}\newcommand{\fL}{\mathcal{L}}

\newcommand{\R}{\mathbb{R}}\newcommand{\E}{\mathbb{E}}
\newcommand{\indot}[2]{\left\langle #1,\, #2\right\rangle}
\newcommand{\enrank}{\mathrm{EnRank}}
\newcommand{\effrank}{\mathrm{EffRank}}

\newcommand{\act}{\mathrm{act}}
\newcommand{\eNTK}[1]{\widehat{K}(#1)}
\newcommand{\relu}{\mathrm{ReLU}}
\newcommand{\mgap}{M_\mathrm{gap}}
\newcommand{\reg}{\mathrm{reg}}
\newcommand{\cls}{\mathrm{cls}}
\newcommand{\pre}{\mathrm{pre}}
\newcommand{\rand}{\mathrm{rand}}
\newcommand{\val}{\mathrm{val}}
\newcommand{\full}{\mathrm{full}}
\newcommand{\onevec}{\mathbf{1}}
\newcommand{\optread}{\textsc{OR}}
\newcommand{\indicator}[1]{\mathbb{I}\qty(#1)}
\newcommand{\emloss}{\widehat{\fL}}
\newcommand{\emgrad}{\widehat{g}}

\title{Predicting Plasticity in Deep Continual Learning: \\A Theoretical Perspective}

\author{%
  Jiuqi Wang \\
  University of Virginia\\
  \texttt{jiuqi@email.virginia.edu} \\
  \And
  Jayanth Srinivasa \\
  Cisco Research \\
  \texttt{jasriniv@cisco.com} \\
  \AND
  Claire Chen \\
  California Institute of Technology\\
  \texttt{clairechen@caltech.edu} \\
  \And
  Shuze Daniel Liu \\
  Purdue University\\
  \texttt{liu4869@purdue.edu} \\
  \And
  Ali Payani\\
  Cisco Research \\
  \texttt{apayani@cisco.com} \\
  \And
  Shangtong Zhang \\
  University of Virginia \\
  \texttt{shangtong@virginia.edu} \\
}

\begin{document}

\maketitle

\begin{abstract}
Deep continual learning requires models to adapt to new tasks without retraining from scratch. 
However, neural networks can lose their ability to adapt to new tasks after training on previous ones, a phenomenon known as loss of plasticity. 
There have been several explanations and diagnostics proposed for plasticity loss. 
Motivated by the philosophy that ``all models are wrong, but some are useful", 
we ask: can existing diagnostics predict a neural network's plasticity? 
In this work, we take a practical view to interpret plasticity as task-dependent trainability, i.e., a neural network's future optimization gain on a target task. 
We first take a theoretical approach, showing, by constructing a few counterexamples, that some widely adopted diagnostics of plasticity, 
including representation rank and neural tangent kernel rank, 
can fail to predict the loss of trainability in both regression and classification settings. 
We instead propose a novel metric, called optimization readiness, which combines gradient strength and gradient reliability. 
We prove that optimization readiness lower bounds one-step optimization gain under standard smoothness assumptions, providing a theoretical guarantee for its predictive power. 
Empirically, we show that across commonly used deep continual learning settings, 
such as Slowly-Changing Regression and Permuted MNIST, 
optimization readiness more reliably ranks checkpoints by their task-dependent trainabilities than prior label-agnostic diagnostics that ignore the labels when computing their metrics, even with substantially fewer samples.
We release our code at \url{https://github.com/Sequential-Intelligence-Lab/Predicting_Plasticity}.
\end{abstract}
\section{Introduction}
Deep continual learning describes a machine learning setting in which a neural network learns from a sequence of tasks, where the learner needs to adapt sequentially while retaining useful knowledge from earlier experience~\citep{aljundi2019continual, parisi2019continual}.
Such continual adaptation is possible only if the neural network model remains trainable,
i.e., it can achieve meaningful improvement on the new task via gradient-based optimization.
Unfortunately, continual adaptation cannot be achieved simply by continuing gradient-based training on newly arriving data, because neural networks can suffer from loss of plasticity~\citep{dohare2024lop}.
Loss of plasticity is the phenomenon in which a previously trained model loses its ability to adapt to new tasks, often performing worse than a randomly initialized network after the same amount of training on the new task.
In severe cases, optimization may stagnate, leaving the learning curve flat and posing a serious challenge to continual adaptation~\citep{abbas2023loss}.

Several diagnostics and mechanisms are proposed for explaining plasticity loss. 
For instance, the effective rank of the representations~\citep{dohare2024lop}, the rank of the empirical neural tangent kernel~\citep{lyle2025disentangling,tang2025mitigating}, the directions of curvature~\citep{lewandowski2024direction}, and the proportion of dead/dormant neurons~\citep{sokar2023dormant,ma2024revisiting}.
Motivated by Box's philosophy that ``all models are wrong, but some are useful,''~\citep{box1976science,box1986empirical} 
we ask: can existing diagnostics predict a neural network's plasticity?
To study this question, we first need to operationalize plasticity.
Unfortunately, plasticity itself is operationalized differently across the literature.
\citet{dohare2024lop,kumar2025maintaining} evaluate plasticity through learning speed.
\citet{wang2026ftle} evaluate plasticity through final performance after a fixed training budget.
\citet{abbas2023loss,lyle2023understanding,nikishin2023deep} evaluate plasticity through performance relative to a randomly initialized model.
A consensus we perceive is that a network is plastic if additional gradient-based optimization produces a meaningful reduction in loss.
Here, by meaningful reduction, we mean a reduction that is significant relative to the initial loss.
We refer to this perceived consensus as trainability, and we view it as a concrete aspect of plasticity in this work.
We will later mathematically formalize trainability as the $k$-step gain, the relative loss reduction achieved after $k$ gradient steps on the target task,
and use it as the target we wish to predict from diagnostic metrics.
Note that this target is by nature task-dependent.

Recent empirical evidence suggests that some commonly used diagnostics are unreliable indicators of plasticity~\citep{lyle2023understanding,lewandowski2024direction, baveja2025unified}. 
However, existing falsification is largely empirical, and there remains a limited theoretical understanding of why such metrics can fail.

In this work, we show that several structural diagnostics can be misleading:
they may assign a favorable value to a neural network checkpoint even when
gradient descent cannot make progress on the relevant task. 
In particular, our theoretical counterexamples focus on label-agnostic structural metrics, such as representation rank and eNTK rank, which depend on the
input distribution and checkpoint but not directly on task labels. 
These negative results suggest that favorable structural diagnostic values alone
are insufficient for predicting future optimization progress. 
We therefore propose an optimization-derived metric that is directly motivated by the
one-step descent dynamics of stochastic gradient descent.

Specifically, our paper makes the following contributions:
\begin{enumerate}
    \item We prove by construction of counterexamples that some widely-adopted diagnostics of plasticity fail to predict the loss of task-dependent trainability in both regression and binary classification settings.
    In particular, we show that label-agnostic metrics, such as the effective rank, the 99\% energy rank of the representation matrix, and the eNTK, can fail even in simple constructions.
    These results show that favorable structural diagnostic values alone do not guarantee future optimization progress.
    \item Motivated by these negative results, we propose Optimization Readiness
    ($\optread$), a lightweight optimization-derived metric designed to predict task-dependent
    trainability. 
    We prove that $\optread$ lower-bounds the one-step optimization gain under standard smoothness assumptions, providing a theoretical guarantee of its predictive power. 
    \item Empirically, we evaluate $\optread$ on Slowly-Changing Regression and Permuted MNIST, two standard testbeds of loss of plasticity~\citep{dohare2024lop}, and demonstrate that it accurately predicts the relative ordering of checkpoints by their task-dependent trainability.
    In our ablation study, we additionally show that $\optread$ is data-efficient, achieving high pairwise ranking accuracy even with 10\%, 1\%, and 0.1\% of the validation data in Slowly-Changing Regression and 10\% and 1\% of the Permuted MNIST validation set.
    The superiority of label-dependent metrics also carry over to out-of-distribution tasks, where we demonstrate that $\optread$ and other label-dependent diagnostics rank task-dependent trainabilities significantly more reliably than their label-agnostics counterparts on a different task family from training.
\end{enumerate}

\section{Preliminaries and Problem Setup}

We study whether a diagnostic computed at a checkpoint can predict its
\emph{trainability} on a newly observed supervised task. A continual
supervised-learning problem is a sequence of tasks
$\fT_1,\fT_2,\ldots$, where each task
$\fT=(\fX,\fY,P_{\fX\fY})$ consists of an input space, a label
space, and a joint distribution over $\fX\times\fY$.
Tasks arrive sequentially, and samples within each task are drawn i.i.d.
from the corresponding distribution. 
A learner is a neural network
$f_\theta:\fX \to \fZ$, and all tasks share a pointwise
loss $\ell:\fZ\times\fY\to \R_{\ge 0}$. 
The population loss on task $\fT$ is
$\fL(\theta;\fT)\doteq \E_{(x,y)\sim P_{\fX\fY}}\qty[\ell(f_\theta(x),y)]$.

\compactpara{Trainability target.}
For a checkpoint $\theta$, our target quantity is the relative loss
reduction obtained by continuing gradient-based optimization on $\fT$.
Let $g(\theta;\fT) \doteq \nabla_\theta \fL(\theta;\fT)$
denote the population gradient. 
For a mini-batch $\fB\doteq\{(x_i,y_i)\}_{i=1}^m$, define
$\emloss_\fB(\theta;\fT) \doteq \frac{1}{m}\sum_{i=1}^m \ell(f_\theta(x_i),y_i)$ as the empirical loss and
$\emgrad_\fB(\theta;\fT) \doteq \nabla_\theta \emloss_\fB(\theta;\fT)$.
Starting from $\theta_0=\theta$, stochastic gradient descent with
learning rate $\eta>0$ follows
\begin{align}
    \theta_{t+1} =
    \theta_t - \eta \emgrad_{\fB_t}(\theta_t;\fT),
    \label{eq: mini-batch gradient descent}
\end{align}
for $t = 0,\ldots,k-1$,
where $\fB_0,\ldots,\fB_{k-1}$ are sampled i.i.d. from $P_{\fX\fY}$.
We define the \(k\)-step gain as
\begin{definition}[mini-batch $k$-step gain]
    \begin{align}
        \fG^{(k)}(\theta;\fT) \doteq
        \begin{cases}
        \displaystyle
        \frac{\fL(\theta;\fT) - \E_{\fB_0,\ldots,\fB_{k-1}}\qty[\fL(\theta_k;\fT)]}
        {\fL(\theta;\fT)}, & \fL(\theta;\fT)>0, \\
        0, & \fL(\theta;\fT)=0 .
        \end{cases}
    \end{align}
    \label{def: mini-batch k-step gain}
\end{definition}

Thus, $\fG^{(k)}$ measures the expected relative improvement after $k$
gradient steps, and we use it as an operational proxy for plasticity.
When the neural network loses its plasticity, we expect the $k$-step gain to be small across different $k$'s.
Normalization in $\fG^{(k)}$ is important because the absolute reduction
in loss can be misleading across checkpoints with different initial losses.
For example, a checkpoint close to convergence may have little room for
absolute improvement, and therefore may exhibit a smaller absolute loss
reduction than a randomly initialized checkpoint. 
This does not necessarily mean that it is less trainable --- after the same number of optimization steps, it may reduce its remaining loss by a comparable, or even larger, proportion.
Normalizing by the initial loss, therefore, makes the gain comparable across
checkpoints at different stages of optimization.

\compactpara{Matrix-rank diagnostics.}
Several existing plasticity diagnostics are based on the spectrum of
representations, neural tangent kernels (NTK,~\citet{jacot2018neural}), or curvature matrices. 
For a matrix $Z\in\R^{m \times n}$, 
let $\sigma_1\ge \cdots \ge \sigma_q \ge 0$, $q=\min\qty{m,n}$, denote its
singular values. 
If $Z \neq 0$, define $p_i=\sigma_i/\sum_j\sigma_j$. 
The effective rank is $\effrank(Z) \doteq \exp\qty(-\sum_{i: p_i>0} p_i\log p_i)$,
with the convention $\effrank(0)=0$. 
For $\tau\in(0,1)$, the $\tau$-energy rank is $\enrank{\tau}(Z) \doteq
    \min\qty{k: \frac{\sum_{i=1}^k \sigma_i^2}{\sum_{j=1}^r \sigma_j^2} \ge \tau}$,
where $r \doteq \rank(Z)$.
Again, we assume $\enrank_{\tau}(0)=0$. 
In this paper, we use $\tau=0.99$.
Let $\phi_\theta(x)\in\R^d$ be the penultimate-layer representation of $f_\theta$. 
For inputs $X \doteq (x_1, \ldots, x_m)$, 
let $\Phi_\theta(X) \in \R^{m\times d}$ be the representation matrix
whose $i$-th row is $\phi_\theta(x_i)$. 
We consider $\effrank\qty(\Phi_\theta(X))$ and $\enrank_{0.99}\qty(\Phi_\theta(X))$
as representation-rank diagnostics.

For scalar-output networks, the empirical neural tangent kernel
(eNTK) on \(X\) is
$\eNTK{X;\theta}_{ij} \doteq \indot{\nabla_\theta f_\theta(x_i)}{\nabla_\theta f_\theta(x_j)}$.
We use $\enrank_{0.99}\qty(\eNTK{X;\theta})$ as the eNTK-rank diagnostic. 
We also compare against a curvature diagnostic, where
for $H(\theta;\fT)\doteq \nabla_\theta^2 \fL(\theta;\fT)$,
we measure $\enrank_{0.99}\qty(H(\theta;\fT))$, or an empirical approximation thereof.

\compactpara{Active-neuron diagnostic.}
For a network with $L$ hidden layers, 
let $a_i(x;\theta)\in\R^{w_i}$ denote the activation vector of hidden layer $i$. 
Given inputs $X\doteq\qty{x_1,\ldots,x_m}$, define the
relative activation score of neuron $j$ in layer $i$ as
$s_{i,j}(X;\theta)\doteq 
    \frac{\frac{1}{m}\sum_{x\in X}\abs{a_i(x;\theta)_j}}
    {\frac{1}{m}\sum_{x\in X}
      \frac{1}{w_i}\sum_{r=1}^{w_i}\abs{a_i(x;\theta)_r} + \varepsilon}$,
where $\varepsilon > 0$ is a small constant to prevent division by zero.
For a threshold \(\tau_\act \ge 0\), the active-neuron fraction is
$A_{\tau_\act}(X;\theta)\doteq
\frac{\sum_{i=1}^L \sum_{j=1}^{w_i}\indicator{s_{i,j}(X;\theta)\ge \tau_\act}}
{\sum_{i=1}^L w_i}$, where $\indicator{\cdot}$ is the indicator function.

\section{Theoretical Falsifications of Structural Plasticity Metrics}
\label{sec:theoretical-falsifications}

In this section, we present theoretical counterexamples to some previously adopted diagnostic metrics of plasticity.
We will show, by construction, that the neural network cannot make any progress at all at a particular checkpoint in some tasks, even though those metrics would assign a favorable value to that checkpoint.
In particular, the structural metrics we study in this section do not use task labels directly, and they exhibit a similar failure mode across the same checkpoint-task construction.
We provide theoretical falsifications for both regression and binary classification tasks.

\compactpara{Shared Construction.}
Fix $n\ge 4$ and $M\ge n$. 
For binary classification, we additionally assume that
$n$ is even. 
Let $\bar X \doteq I_n$ and $\bar x_i \doteq e_i$ for $i=1,\dots,n$,
and consider the two-layer ReLU network
$f_\theta(x) \doteq v^\top \relu(Wx)$, where
$\theta \doteq (v,W)$,
$ v\in\R^M$,
and $W\in\R^{M\times n}$.   
Then, for any checkpoint $\theta$, define the representation matrix on the full
support by
$\Phi_\theta(\bar X)_{i,j} \doteq \relu(w_j^\top \bar x_i)=\relu(W_{j,i}),
\quad\Phi_\theta(\bar X)\in\R^{n\times M}$.
Given a label vector $\bar Y \doteq \mqty[\bar y_1 & \cdots & \bar y_n]^\top$ and a
differentiable pointwise loss $\ell(z,y)$, 
let $\bar\fT(\bar Y,\ell)$ denote the full-support task that samples uniformly from
$\qty{(\bar x_i,\bar y_i)}_{i=1}^n$. 
Then, its population loss is
$\fL(\theta;\bar\fT(\bar{Y}, \ell)) 
= \frac1n\sum_{i=1}^n \ell(f_\theta(\bar x_i),\bar y_i)$.
We analyze the full-support gradient descent
$\theta_{s+1}=\theta_s-\eta\nabla_\theta\fL(\theta_s;\bar\fT),
\quad
\theta_0=\theta$,
and the corresponding full-support $k$-step gain
\begin{align}
    \fG_{\full}^{(k)}(\theta;\bar\fT(\bar{Y},\ell))
    \doteq
    \begin{cases}
    \dfrac{\fL(\theta;\bar\fT(\bar{Y},\ell))-\fL(\theta_k;\bar\fT(\bar{Y},\ell))}
          {\fL(\theta;\bar\fT(\bar{Y},\ell))},
    & \fL(\theta;\bar\fT(\bar{Y},\ell))>0,\\
    0, & \fL(\theta;\bar\fT(\bar{Y},\ell))=0.
    \end{cases}
\end{align}
    
This is a favorable setting for the learner because the update uses the exact population gradient rather than a noisy mini-batch gradient.
However, a zero full-support gradient does not necessarily imply that individual-example gradients vanish, as they may cancel in the full-support average. 
Thus, zero full-support gain does not in general imply zero mini-batch gain.

The following three lemmas are the only common ingredients needed for both
counterexamples.

\begin{lemma}[one step update]
\label{lem:zero-output-update}
For a task \(\bar\fT(\bar Y,\ell)\), define $q(\theta)\doteq \mqty[q_1(\theta),\dots,q_n(\theta)]^\top$, where
$q_i(\theta)\doteq\left. \frac{\partial}{\partial z}\ell(z,\bar y_i) \right|_{z=f_\theta(\bar x_i)}$.
If $\theta=(0,W)$, then
$\nabla_v\fL(\theta;\bar\fT(\bar{Y},\ell))
=\frac1n \Phi_\theta(\bar X)^\top q(\theta)$
and 
$\nabla_W\fL(\theta;\bar\fT(\bar{Y},\ell))=0$.
Consequently, if
$\Phi_\theta(\bar X)^\top q(\theta)=0$,
then full-support gradient descent leaves $\theta$ unchanged for every
learning rate $\eta>0$, and hence
$\fG_{\full}^{(k)}(\theta;\bar\fT(\bar{Y},\ell))=0$
for all $k\ge 1$.
Moreover, one full-support gradient step from $\theta_0\doteq(0,W)$ satisfies
$v_1=-\frac{\eta}{n}\Phi_{\theta_0}(\bar X)^\top q(\theta_0)$
and $W_1=W$, and therefore
$f_{\theta_1}(\bar X)= -\eta G(W)q(\theta_0)$, where
$G(W)\doteq\frac1n\Phi_{\theta_0}(\bar X)\Phi_{\theta_0}(\bar X)^\top$ is the Gram matrix.
\end{lemma}
We give the proof in Appendix~\ref{proof:zero-output-update}.

\begin{lemma}[rank metrics from a positive semi-definite activation matrix]
\label{lem:psd-rank-bookkeeping}
Let \(B\in\R^{n\times n}\) be symmetric positive semidefinite with nonnegative
entries. 
Fix \(\alpha>0\) and suppose \(v=0\) and
$\Phi_\theta(\bar X)=\alpha\mqty[ B & 0_{n\times(M-n)}]$.
If the nonzero eigenvalues of \(B\) are
\(b_1\ge b_2\ge\cdots\ge b_r>0\), where $r$ is the rank of $B$, then the nonzero singular values of
\(\Phi_\theta(\bar X)\) are
$\alpha b_1,\dots,\alpha b_r$,
and the nonzero singular values of the eNTK \(\eNTK{\bar X; \theta}\) on \(\bar X\)
are $\alpha^2 b_1^2,\dots,\alpha^2 b_r^2$.
\end{lemma}
The proof is in Appendix~\ref{proof: psd-rank-bookkeeping}.

\begin{lemma}[random ReLU Gram concentration]
\label{lem:random-gram-concentration}
Let \(W_{\rand}\in\R^{M\times n}\) have i.i.d. entries sampled from
\[
    \begin{cases}
    +\sqrt{2/n}, & \text{with probability }1/2,\\
    -\sqrt{2/n}, & \text{with probability }1/2.
    \end{cases}
\]
Let \(\theta_{\rand} \doteq (0,W_{\rand})\), and define
$G_{\rand} \doteq \frac1n \Phi_{\theta_{\rand}}(\bar X) \Phi_{\theta_{\rand}}(\bar X)^\top$.
Using $\onevec\in\R^n$ to denote the all-one vector,
we have
$\bar G\doteq\E[G_{\rand}]=\lambda(J_n+I_n)$,
where $\lambda\doteq \frac{M}{2n^2}$ and \(J_n=\onevec_n\onevec_n^\top\). 
Moreover, there is a universal constant
$C>0$, such that for every $\delta\in(0,1)$ and every $\rho>1$, if
$M\ge C n^2\rho\log\qty(\frac{2n}{\delta})$,
then, with probability at least \(1-\delta\),
$\norm{G_{\rand}-\bar G}_2\le\frac{\lambda}{\sqrt{\rho}}$.
\end{lemma}
We provide the proof in Appendix~\ref{proof: random-gram-concentration}.

\compactpara{Regression Counterexample.}
We now construct a counterexample for the regression task, showing that there exists a checkpoint-task combination in which the model does not learn, even though the metric values are near-optimal.
In addition, we present a random initialization of the hidden-layer weights that guarantees learning with high probability.
\begin{theorem}
\label{thm:regression-counterexample-condensed}
Let $\bar Y_{\mathrm{reg}}\doteq e_{n-1}-e_n$ and 
$\ell_{\mathrm{sq}}(z,y)\doteq\frac12(z-y)^2$.
We use
\(\bar\fT_{\mathrm{reg}}\doteq \bar\fT(\bar Y_{\mathrm{reg}},\ell_{\mathrm{sq}})\)
as a shorthand.
There exists a checkpoint \(\theta_{\pre}\) such that the following hold.
\begin{enumerate}[label=(\roman*)]
    \item The checkpoint has high representation and eNTK rank. 
    In particular,
    \begin{align}
        \begin{cases}
            \enrank_{0.99}\qty(\Phi_{\theta_{\pre}}(\bar X)) 
            &= \left\lceil 0.99(n-2)\right\rceil;\\
            \effrank\qty(\Phi_{\theta_{\pre}}(\bar X))
            &= n-2;\\
            \enrank_{0.99}\qty(\eNTK{\bar X; \theta_{\pre}})
            &= \left\lceil 0.99(n-2)\right\rceil.
        \end{cases}
    \end{align}
    \item Yet, the checkpoint is completely stuck under both full-support
    gradient descent and mini-batch SGD: $\fL(\theta_{\pre};\bar\fT_\reg) = \frac1n$ and
    $\fG_{\full}^{(k)}(\theta_{\pre};\bar\fT_\reg)
    = \fG^{(k)}(\theta_{\pre};\bar\fT_\reg) = 0$
    for all $k\ge 1$, $\eta > 0$, and mini-batch sizes $m\ge 1$.
    \item In contrast, let \(\theta_{\rand}\doteq(0,W_{\rand})\) be the random
    initialization in Lemma~\ref{lem:random-gram-concentration}. 
    Setting $\eta=\frac{2n^2}{M}$,
    there exists a universal constant $C > 0$, such that for every
    \(\delta\in(0,1)\) and \(\mgap>1\), if
    $M \ge C n^2\mgap\log\qty(\frac{2n}{\delta})$,
    then, with probability at least \(1-\delta\),
    $\fG_{\full}^{(1)}(\theta_{\rand};\bar\fT_\reg) \ge 1-\frac1{\mgap}$.
\end{enumerate}
\end{theorem}
The proof is in Appendix~\ref{proof:regression-counterexample-condensed}.

\compactpara{Binary Classification Counterexample}
Here, we present the counterexample for the binary classification task, mirroring the goal of the regression counterexample.
\begin{theorem}
\label{thm:classification-counterexample-condensed}
Assume \(n\ge 4\) is even. 
Let
\[
    (\bar Y_{\mathrm{cls}})_i
    =
    \begin{cases}
        1, & i=1,\dots,n/2,\\
        0, & i=n/2+1,\dots,n,
    \end{cases}
\]
and define the centered sign vector
$\bar S_{\mathrm{cls}}\doteq2\bar Y_{\mathrm{cls}}-\onevec_n\in\{-1,+1\}^n$.
Let the binary logistic loss under the \(0/1\)-label convention be
$\ell_{\mathrm{log}}(z,y)\doteq\log(1+\exp(z))-yz$,
for $y\in\qty{0,1}$.
Equivalently, we have
$\ell_{\mathrm{log}}(z,y) = \log\qty(1+\exp\qty(-(2y-1)z))$.
Define $\bar\fT_{\mathrm{cls}}\doteq\bar\fT(\bar Y_\cls,\ell_{\mathrm{log}})$ as a shorthand.
There exists a checkpoint \(\theta_{\pre}\), such that the following hold.

\begin{enumerate}[label=(\roman*)]
    \item The checkpoint has high representation and eNTK rank.
    In particular,
    \begin{align}
        \begin{cases}
            \enrank_{0.99}\qty(\Phi_{\theta_{\pre}}(\bar X))
            &= \left\lceil 0.99(n+2)-3\right\rceil;\\
            \effrank\qty(\Phi_{\theta_{\pre}}(\bar X))
            &= \frac{n}{2^{2/n}};\\
            \enrank_{0.99}\qty(\eNTK{\bar X; \theta_{\pre}})
            &=\left\lceil 0.99(n+14)-15\right\rceil.
        \end{cases}
    \end{align}
    \item Yet, the checkpoint is completely stuck under a full-support gradient
    descent, where $\fL(\theta_{\pre};\bar\fT_{\mathrm{cls}})=\log 2$
    and $ \fG_{\full}^{(k)}(\theta_{\pre};\bar\fT_{\mathrm{cls}}) = 0$
    for all $k\ge1$ and $\eta>0$.
    \item In contrast, let \(\theta_{\rand}\doteq(0,W_{\rand})\) be the random
    initialization in Lemma~\ref{lem:random-gram-concentration}. 
    For \(\mgap>1\), define
    $\gamma_{\mgap}\doteq\log\qty(\frac{2\mgap}{\log 2})$ and
    $\eta\doteq\frac{4\gamma_{\mgap}n^2}{M}$.
    There exists a universal constant \(C>0\), such that for every
    \(\delta\in(0,1)\), if
    $M \ge C n^2\gamma_{\mgap}^2\mgap^2\log\qty(\frac{2n}{\delta})$,
    then, with probability at least \(1-\delta\), it holds that
    $\fG_{\full}^{(1)}(\theta_{\rand};\bar\fT_{\mathrm{cls}}) \ge 1-\frac1{\mgap}$.
\end{enumerate}
\end{theorem}
We give our proof in Appendix~\ref{proof:classification-counterexample-condensed}.

Our constructions are intentionally stylized as they are not meant to model
the typical trajectory of practical continual learning systems, but
to isolate a failure mode showing that favorable structural diagnostic
values alone do not guarantee optimization progress on a target task.

The two theorems show that high representation rank and high eNTK rank do not
by themselves imply trainability on a target task. 
In both cases, the failure is
caused by a label-dependent null direction, where the checkpoint representation has
large rank on the inputs, but the task gradient is exactly orthogonal to the
directions that the output layer can use at initialization.
This observation motivates an optimization-derived diagnostic that depends directly on the gradient signal and its reliability, rather than only on structural properties of the checkpoint.


\section{Optimization Readiness}
In this section, we propose Optimization Readiness ($\optread$) as a predictive metric of trainability.
$\optread$ is the product of \emph{gradient strength} and \emph{gradient reliability}, where the former measures if there is a gradient signal at all, and the latter indicates the signal-to-noise ratio.
Unlike diagnostics motivated primarily by structural or geometric properties of the checkpoint, 
OR is derived from the expected descent dynamics of stochastic gradient descent. 
This connection allows us to prove a lower bound on one-step gain under a smoothness assumption and an appropriate learning rate.

Let $\fT \doteq (\fX, \fY, P_{\fX\fY})$ be an arbitrary supervised learning task and $\fB$ be a mini-batch of size $m$, where each input-label pair from $\fB$ is drawn i.i.d. from $P_{\fX\fY}$. 
We formally state the assumptions and define $\optread$ as follows.
\begin{assumption}[unbiasedness]
    For any checkpoint $\theta$, it holds that
    $\E_{\fB}\qty[\emgrad_\fB(\theta;\fT)] = g(\theta;\fT)$.
    \label{assumption: unbiased estimator}
\end{assumption}
\begin{assumption}[smoothness]    
    The population loss
    $\fL(\theta; \fT)$ is differentiable with respect to $\theta$ and $\beta$-smooth, such that for all $\theta_1$, $\theta_2$, it holds that
    $\norm{\nabla_\theta \fL(\theta_1;\fT) - \nabla_\theta \fL(\theta_2; \fT)}_2
        \le \beta\norm{\theta_1 - \theta_2}_2$
    for some $\beta > 0$.
    \label{assumption: smoothness}
\end{assumption}

The Euclidean norm of the gradient vector is a straightforward way to monitor the magnitude of the gradient signal.
When the gradient norm becomes zero, gradient descent stops updating the parameters.
In addition, gradient norm alone can be misleading because it does not distinguish checkpoints at different stages of convergence.
Hence, we normalize the squared gradient norm by the population loss and define gradient strength as
\begin{definition}[gradient strength]
    $S(\theta; \fT) \doteq \frac{\norm{g(\theta;\fT)}_2^2}{\fL(\theta;\fT)}$.
\end{definition}
The noise in gradient estimation can also dictate the progress of gradient descent, as discussed in~\citet{baveja2025unified}.
Gradient descent may fail to reduce population loss if the gradient signal is overwhelmed by noise, even when the gradient is strong.
In light of this observation, we define gradient reliability as
\begin{definition}[gradient reliability]
    $R(\theta;\fT) \doteq
    \frac{\norm{g(\theta;\fT)}_2^2}{\norm{g(\theta;\fT)}_2^2 + V_\fB(\theta;\fT)}$,
    where $V_\fB(\theta;\fT) \doteq \E_\fB\qty[\norm{\emgrad_\fB(\theta;\fT) - g(\theta;\fT)}_2^2]$.
\end{definition}
Notably, if Assumption~\ref{assumption: unbiased estimator} holds, then 
$R(\theta;\fT) = \frac{\norm{g(\theta;\fT)}_2^2}{\norm{g(\theta;\fT)}_2^2 + \qty(\E_\fB\qty[\norm{\emgrad_\fB(\theta;\fT)}_2^2] - \norm{g(\theta;\fT)}_2^2)} = \frac{\norm{g(\theta;\fT)}_2^2}{\E_\fB\qty[\norm{\emgrad_\fB(\theta;\fT)}_2^2]}$.
We assume the mini-batch size to estimate $R(\theta;\fT)$ is consistent with the gradient descent procedure.
We therefore define
\begin{definition}[Optimization Readiness]
    \begin{align}
        \optread(\theta;\fT) \doteq 
        \begin{cases}
            S(\theta;\fT)R(\theta;\fT), & \fL(\theta;\fT) > 0 \text{ and } \E_\fB\qty[\norm{\emgrad_\fB(\theta;\fT)}_2^2] > 0,\\
            0, & \text{otherwise}.
        \end{cases}
    \end{align}
    \label{def: optimization readiness}
\end{definition}
\begin{theorem}[Optimization Readiness Lower-Bounds One-Step Gain]
    Let $\fT \doteq (\fX, \fY, P_{\fX\fY})$ be any supervised learning task and $\theta$ be an arbitrary parameterization of the neural network.
    Let $\theta_0 \doteq \theta$ and $\theta_1$ be the parameters after one step of mini-batch gradient descent as in~\eqref{eq: mini-batch gradient descent} with $\fB_0 \doteq \qty{(x_1,y_1), \dots, (x_m, y_m)}$, where $ (x_i, y_i) \overset{\text{i.i.d.}}{\sim} P_{\fX\fY}$ for $i = 1, \dots, m$.
    Suppose Assumption~\ref{assumption: unbiased estimator} and \ref{assumption: smoothness} hold.
    If $0 < \eta < R(\theta;\fT)/\beta$, then it holds that
    $\fG^{(1)}(\theta;\fT) \ge  \frac{\alpha(1 - \alpha)}{\beta}\optread(\theta;\fT)$,
    where $\alpha = \eta\beta / R(\theta;\fT)$ and $\beta$ is the smoothness constant.
    \label{thm: optimization readiness one-step gain bound}
\end{theorem}
The proof can be found in Appendix~\ref{proof: optimization readiness one-step gain bound}.

\section{Experiments}
In this section, we empirically verify if the metrics of plasticity can reliably predict $k$-step gains.
We use Slowly-Changing Regression (SCR) for continual regression and Permuted MNIST (P-MNIST) for continual classification as testbeds~\citep{dohare2024lop}.
For SCR, the features are binary vectors of dimension $u + v$, where the first $u$ dimensions are slowly-changing bits, and the remaining $v$ dimensions are random bits.
Suppose each task has $N$ data points, the $v$ random bits are randomly sampled for each data point, while the $u$ slowly changing bits stay constant for $N$ steps.
After $N$ steps, one of the $u$ bits is uniformly randomly selected and flipped from 1 to 0 or from 0 to 1, thereby creating a slowly shifting distribution in the input space across tasks.
The labels are generated by a linear threshold unit (LTU) network~\citep{dohare2024lop} with its weights randomly initialized to be $+1$ or $-1$ and fixed across all tasks.
Regarding P-MNIST, we generate a random permutation of the pixels for each task and apply it to all images within that task, while keeping the original labels.
We flatten the images row-wise to form the feature vectors.

We compare $\optread$ against representative structural diagnostics, including
representation-rank metrics, eNTK rank, and active-neuron fraction.
We also include Hessian rank as a task-dependent curvature baseline,
allowing us to test whether OR remains useful beyond purely structural
input-checkpoint diagnostics.

\compactpara{Data Generation.}
We generate separate training and validation tasks for both SCR and P-MNIST.
The training tasks are designed to induce loss of plasticity through
continual adaptation to non-stationary data streams, while the validation
tasks are used to evaluate the predictive power of plasticity metrics on
held-out stationary tasks.
Detailed dataset statistics are summarized in Table~\ref{tab:data-generation}.

\begin{table}[t]
\centering
\small
\caption{Dataset generation statistics for Slowly-Changing Regression and Permuted MNIST.}
\label{tab:data-generation}
\begin{tabular}{lcc}
\toprule
 & Slowly-Changing Regression & Permuted MNIST \\
\midrule
Training tasks per sequence & 1,000 & 800 \\
Training samples per sequence & 1,000,000 & 8,000,000 \\
Validation samples per task & 10,000 & 10,000 \\
Training sequences & 20 & 20 \\
Validation tasks & 30 & 10 \\
Task shift mechanism &
Bit flipping every 1,000 samples &
Random pixel permutation per task \\
\bottomrule
\end{tabular}
\end{table}

\compactpara{Checkpoint Generation.}
We initialize and train a multi-layer perceptron (MLP) with $\relu$ activation on each training sequence of SCR and P-MNIST.
The purpose of training is to obtain a set of checkpoints with varying trainabilities for later testing of the metrics.
The MLP for SCR has 2 hidden layers of 5 units, whereas the MLP for P-MNIST has 3 hidden layers of 100 units.
We employ mean-squared error for SCR and cross-entropy loss for P-MNIST without regularization.
We use Adam~\citep{kingma2014adam} as the optimizer with batch size $m=1$.
The learning rates are set to $\eta = 0.01$ and $\eta=0.003$ for SCR and P-MNIST, respectively.
Figure~\ref{fig:learning curve} shows the learning curves for both tasks.
There is a clear sign of plasticity loss in both scenarios.
As we save a checkpoint every 5 tasks, we collect a wide range of checkpoints with varying levels of trainability.

\compactpara{$k$-Step Gain and Metric Estimation.}
Given a validation task $\fT_\val$ and a checkpoint $\theta$, we estimate
the $k$-step gain by running stochastic gradient descent for $k$ steps
from $\theta$ on mini-batches sampled from the validation set, and then
measuring the resulting relative reduction in empirical loss.
We repeat this procedure over multiple independent optimization rollouts
to approximate the expectation in Definition~\ref{def: mini-batch k-step gain}.

We estimate $\optread$ using empirical counterparts of its two components.
The population gradient is approximated by the
gradient of the full validation set, 
while the expected mini-batch squared gradient norm is estimated
using independently sampled mini-batches from the validation set.
Together, these quantities give empirical estimates of gradient strength,
gradient reliability, and $\optread$.

For the baseline diagnostics, we compute representation-rank metrics,
active-neuron fraction, eNTK rank, and Hessian-rank metrics on the same
validation data whenever applicable.
For SCR, the network is small enough to compute the exact Hessian and
eNTK rank.
For P-MNIST, we omit the eNTK rank because the vector-valued output makes
the corresponding empirical kernel less directly comparable to the
scalar-output case, and we approximate the Hessian rank using the
Gram-matrix estimator of~\citet{lewandowski2024direction}.
Full estimator definitions and implementation details are provided in
Appendix~\ref{app:estimation-details}.

\compactpara{Qualitative Analysis.}
We first conduct a qualitative analysis to examine the relationship
between the metric values and their corresponding $k$-step gains.
Figure~\ref{fig:scatter plot} displays the metric values against the
ground-truth $k$-step gains.
We observe an approximately linear relationship between the $1$-step
gain and $\optread$, whereas the other metrics do not exhibit a clear
positive correlation with trainability.
Even for $k=100$, $\optread$ maintains a strong positive relationship
with the $k$-step gain, suggesting that it remains informative beyond
the one-step setting considered in our theory.
However, we note that some points lie on the x-axis for $\optread$ in P-MNIST. We attribute these to false positives, where low-trainability checkpoints receive high metric values. Nonetheless, this is not surprising, as we do not guarantee that the assumptions in our theory hold in these empirical tests, and $\optread$'s correlation with trainability is still considerably clearer than the others.

Next, we visualize the evolution of the $k$-step gains and the metric
values throughout continual training in
Figure~\ref{fig:trajectory plot}, where the top three rows correspond to
the $k$-step gains and the bottom row shows the metric trajectories.
Among all metrics, the trajectory of $\optread$ aligns most closely with
the evolution of the $k$-step gains, particularly in SCR.
Interestingly, trainability is not maximized near random initialization.
Instead, the $k$-step gains initially increase rapidly before gradually declining.
We conjecture that this phenomenon may be related to forward
transfer~\citep{lopez2017gradient}, where previously learned tasks
temporarily facilitate adaptation to similar future tasks before the effects of plasticity loss begin to dominate.

\compactpara{Quantitative Analysis.}
We next quantitatively evaluate whether each metric preserves the ordering
of checkpoints by trainability. 
Let $\omega(\theta;\fT)$ denote the value of a diagnostic metric at
checkpoint $\theta$ on task $\fT$.
For two checkpoints $\theta_1$ and $\theta_2$, a metric ranks them correctly
if its ordering agrees with the ordering induced by the $k$-step gain, i.e.,
$\qty(\omega(\theta_1;\fT)-\omega(\theta_2;\fT))\qty(\fG^{(k)}(\theta_1;\fT)-\fG^{(k)}(\theta_2;\fT)) > 0$.
We therefore measure pairwise ranking accuracy, defined as the fraction of
checkpoint pairs whose metric ordering agrees with their ordering under
the estimated ground-truth $k$-step gain:
\begin{align}
    \frac{
    \sum_{(\theta_1, \theta_2) \in \tilde{\Theta}\times\tilde{\Theta}}
    \indicator{
    \Delta_\omega(\theta_1,\theta_2;\fT)
    \Delta_{\fG^{(k)}}(\theta_1,\theta_2;\fT) > 0
    }
    \indicator{
    \Delta_{\fG^{(k)}}(\theta_1,\theta_2;\fT) \neq 0
    }}
    {
    \sum_{(\theta_1, \theta_2) \in \tilde{\Theta}\times\tilde{\Theta}}
    \indicator{
    \Delta_{\fG^{(k)}}(\theta_1,\theta_2;\fT) \neq 0
    }
    },
\end{align}
where $\tilde{\Theta}$ is the set of recorded checkpoints,
$\Delta_\omega(\theta_1,\theta_2;\fT)
\doteq \omega(\theta_1;\fT)-\omega(\theta_2;\fT)$, and
$\Delta_{\fG^{(k)}}(\theta_1,\theta_2;\fT)
\doteq \fG^{(k)}(\theta_1;\fT)-\fG^{(k)}(\theta_2;\fT)$.
This quantity is closely related to Kendall-style rank correlation
measures~\citep{kendall1938new}, but has a direct probabilistic
interpretation as the probability that a metric correctly ranks two
checkpoints by trainability.

We collect 201 checkpoints per SCR run and 161 per P-MNIST run.
Table~\ref{tab:pairwise-ranking} reports the pairwise ranking accuracies
across tasks and horizons.
$\optread$ achieves the highest ranking accuracy for both tasks and all
values of $k$, with the largest margin appearing in SCR.

\begin{table}[t]
\centering
\small
\caption{Pairwise ranking accuracy (mean $\pm$ standard error) across tasks.}
\label{tab:pairwise-ranking}
\begin{tabular}{llccc}
\toprule
Task & Metric & $k=1$ & $k=10$ & $k=100$ \\
\midrule

\multirow{6}{*}{\shortstack[l]{Slowly-Changing\\Regression}}
& OR 
& \textbf{0.987 $\pm$ 0.001} 
& \textbf{0.988 $\pm$ 0.001} 
& \textbf{0.977 $\pm$ 0.002} \\

& Repr Effective Rank 
& 0.698 $\pm$ 0.001 
& 0.700 $\pm$ 0.001 
& 0.704 $\pm$ 0.001 \\

& Repr 99\% Energy Rank 
& 0.689 $\pm$ 0.002 
& 0.690 $\pm$ 0.002 
& 0.694 $\pm$ 0.002 \\

& eNTK 99\% Energy Rank 
& \underline{0.765 $\pm$ 0.002} 
& \underline{0.768 $\pm$ 0.002} 
& \underline{0.779 $\pm$ 0.003} \\

& Hessian 99\% Energy Rank 
& 0.734 $\pm$ 0.003 
& 0.737 $\pm$ 0.003 
& 0.746 $\pm$ 0.005 \\

& Active Neuron Fraction 
& 0.671 $\pm$ 0.007 
& 0.673 $\pm$ 0.007 
& 0.679 $\pm$ 0.007 \\

\midrule

\multirow{5}{*}{\shortstack[l]{Permuted\\MNIST}}
& OR 
& \textbf{0.906 $\pm$ 0.005} 
& \textbf{0.905 $\pm$ 0.005} 
& \textbf{0.899 $\pm$ 0.004} \\

& Repr Effective Rank 
& \underline{0.858 $\pm$ 0.002} 
& \underline{0.864 $\pm$ 0.003} 
& \underline{0.873 $\pm$ 0.002} \\

& Repr 99\% Energy Rank 
& 0.828 $\pm$ 0.003 
& 0.834 $\pm$ 0.003 
& 0.844 $\pm$ 0.003 \\

& Hessian 99\% Energy Rank
& 0.821 $\pm$ 0.003 
& 0.753 $\pm$ 0.007 
& 0.717 $\pm$ 0.007 \\

& Active Neuron Fraction 
& 0.346 $\pm$ 0.003 
& 0.377 $\pm$ 0.003 
& 0.391 $\pm$ 0.003 \\

\bottomrule
\end{tabular}
\end{table}

\compactpara{Subsampling Ablation.}
The preceding results estimate all metrics using the full validation set.
However, access to large held-out datasets may be unrealistic in practical
continual learning scenarios.
We therefore evaluate whether $\optread$ remains predictive when estimated
from substantially fewer samples.
For SCR, we estimate $\optread$ using 10\%, 1\%, and 0.1\% of the
validation set, corresponding to 1,000, 100, and 10 data points,
respectively.
For P-MNIST, we use 10\% and 1\% of the validation set, corresponding to
1,000 and 100 data points.
In all subsampling experiments, we estimate $\optread$ using only 16
mini-batches of size $m=1$.

Figure~\ref{fig:subsample bar plot} summarizes the results.
In SCR, $\optread$ remains the most accurate metric even when estimated
from only 0.1\% of the validation data, suggesting that the full-data
configuration used above is conservative.
In P-MNIST, the accuracy of $\optread$ decreases under subsampling,
especially for larger $k$, and can fall below representation-rank metrics
computed using the full validation set.
Nevertheless, it remains the strongest predictor among the metrics
compared under the same data budget.
This observation also suggests that $\optread$ can remain informative even
when estimated using mini-batches smaller than those used in the
$k$-step gain rollouts.

\compactpara{Out-of-domain Generalization.}
So far, we have validated the checkpoints' trainability within the same domain.
A natural question is: Do the metrics predict the future trainability of checkpoints on task B after they are trained on task A?
To answer this question, we employ the Fashion-MNIST dataset~\citep{xiao2017fashionmnist}, which has the same input and output dimensions as P-MNIST.
We construct 10 validation tasks, each with 1,000 random samples from the dataset.
Then, we apply the same permutation and flattening scheme as P-MNIST to process the images.
Without further training, we use the checkpoints generated from learning P-MNIST directly on the permuted Fashion-MNIST validation set.
We measure whether each metric computed on the target tasks preserves
checkpoint ordering by estimated one-step gain, using the pairwise ranking
accuracy defined above.
Furthermore, we also include individual components of $\optread$, such as the unnormalized squared gradient norm $\norm{g(\theta;\fT_\val)}_2^2$, gradient strength ($S$) alone, and gradient reliability ($R$) alone in the analysis.
Table~\ref{tab:ood-pairwise-ranking} summarizes the results.
Overall, label-dependent diagnostics considerably outperform label-agnostics metrics.
Among them, $\optread$ has a slight edge.
This result is consistent with our theoretical insights.

\begin{table}[htbp]
\centering
\small
\caption{Mean pairwise ranking accuracy for one-step gain on permuted Fashion-MNIST.}
\label{tab:ood-pairwise-ranking}
\begin{tabular}{lc}
\toprule
Metric & Pairwise ranking accuracy ($k=1$) \\
\midrule
Optimization Readiness ($\optread$) & \textbf{0.971 $\pm$ 0.001} \\
Squared Gradient Norm ($\norm{g(\theta;\fT_\val)}_2^2$) & 0.965 $\pm$ 0.001\\
Gradient Strength ($S$) & \underline{0.969} $\pm$ 0.001\\
Gradient Reliability ($R$) & 0.950 $\pm$ 0.001\\
Repr Effective Rank & 0.889 $\pm$ 0.001\\
Repr 99\% Energy Rank & 0.859 $\pm$ 0.001\\
Hessian 99\% Energy Rank & 0.852 $\pm$ 0.001\\
Active Neuron Fraction & 0.534 $\pm$ 0.008\\
\bottomrule
\end{tabular}
\end{table}

\section{Related Work}
\compactpara{Explanations and Diagnostics of Loss of Plasticity.}
Understanding and diagnosing loss of plasticity has been an active area
of research in both continual supervised and reinforcement learning
settings~\citep{dohare2024lop,lyle2025disentangling}.
A prominent line of work studies the structural properties of neural
representations.
In particular, representation collapse has been proposed as a key
mechanism underlying plasticity loss, where metrics such as the effective
rank and the 99\% energy rank of the representation are used to explain
and detect degraded trainability~\citep{dohare2023maintaining,dohare2024lop}.
Another widely studied class of diagnostics focuses on spectral
properties of the neural tangent kernel (NTK), including its rank,
sharpness, and condition number~\citep{lyle2025disentangling,tang2025mitigating},
where ill-conditioned NTKs are associated with optimization difficulty.
Loss of curvature directions, reflected by a collapse of the Hessian
rank, has also been proposed as an explanation for plasticity
loss~\citep{lewandowski2024direction,prakash2026spectral}.
More recently,~\citet{wang2026ftle} study finite-time Lyapunov exponents
(FTLEs), arguing that both excessively small and excessively large FTLEs
can impair plasticity.
On a more microscopic level, dormant or dead neurons are also used as
indicators of plasticity loss~\citep{sokar2023dormant,ma2024revisiting}.

While these works provide valuable explanations and empirical
correlations, many existing diagnostics are motivated primarily through
structural or geometric properties of the checkpoint, rather than through
a direct characterization of future optimization progress.
It therefore remains unclear whether such diagnostics can reliably predict
future trainability.
Our work studies this predictive question directly by evaluating
diagnostics against future optimization gain and by proposing an
optimization-derived metric with theoretical and empirical support.

\compactpara{Empirical Falsifications and Limitations of Existing Diagnostics.}
Several works have demonstrated that existing diagnostics cannot fully
explain trainability in neural networks.
\citet{lyle2023understanding} show that metrics such as weight norm,
weight rank, dead units, and feature rank can exhibit inconsistent
correlations with plasticity across settings.
\citet{lewandowski2024direction} construct empirical counterexamples
showing that update norm, neuron dormancy, representation rank, and
weight norm can all be misleading indicators of plasticity.
Likewise,~\citet{baveja2025unified} demonstrate that individual
indicators such as Hessian rank, unit sign entropy, and gradient norm
cannot fully characterize trainability on their own.

These works suggest that existing diagnostics are incomplete, but their
analysis is primarily empirical.
Our work complements this line of research by providing theoretical
counterexamples that explicitly characterize failure modes of several
widely used diagnostics.

\section{Discussion and Conclusion}
In this work, we study plasticity in deep continual learning from the
perspective of future optimization progress.
We propose $k$-step gain as an operational proxy for trainability and use
it to evaluate whether existing diagnostics can predict how much a
checkpoint will improve after further optimization on a target task.

Our theoretical results show that high structural diversity alone does not
guarantee trainability.
In particular, we construct regression and binary classification tasks on
which metrics based on representation rank and eNTK rank assign favorable values to
a checkpoint, even though full-support gradient descent cannot reduce
the loss from that checkpoint.
On the same tasks, we also show that randomly initialized weights are
likely to make substantial one-step progress, highlighting a concrete gap
between structural diagnostic values and actual optimization gain.

Motivated by this gap, we propose optimization readiness, an
optimization-derived metric combining gradient strength and gradient
reliability.
Under a smoothness assumption and an appropriate learning rate, we prove that the optimization readiness lower bounds the one-step gain.
Empirically, we find that $\optread$ reliably ranks checkpoints by their
1-, 10-, and 100-step gains on both Slowly-Changing Regression and Permuted MNIST, outperforming several
commonly used diagnostics.
Our subsampling ablation further suggests that $\optread$ can remain
informative even when estimated with limited validation data and a small
number of Monte Carlo samples.

Our study deliberately focuses on simple yet standard continual learning
testbeds in order to isolate the predictive behavior of plasticity
diagnostics.
This controlled setting allows us to connect theory and empirical
evaluation cleanly, but it also leaves several important directions for
future work.
First, our counterexamples use $\relu$ networks with one hidden layer.
Extending the analysis to deeper networks and other activation functions would
strengthen the theoretical scope.
Second, our empirical evaluation focuses on supervised continual learning.
Thus, testing whether the same observations hold in continual reinforcement
learning and larger-scale architectures and problems is an important next step.
Overall, our results suggest that diagnostics of plasticity should be
evaluated not only as retrospective explanations of observed training
behavior, but also as prospective predictors of future optimization gain.

\acksection
This work is supported in part by the US National Science Foundation under the awards III-2128019, SLES-2331904, and CAREER-2442098, the Commonwealth Cyber Initiative's Central Virginia Node under the award VV-1Q26-001, VV-1Q27-007, and VV-1Q27005, 4VA at UVA, a Cisco Faculty Research Award, and a Google Research Award.

\bibliographystyle{plainnat}
\bibliography{reference}

\appendix

\section{Math Background}
\begin{theorem}
    \label{thm: matrix Bernstein inequality}
    (Matrix Bernstein inequality)
    Let $M_1, \dots, M_k$ be independent random symmetric matrices such that $M_i \in \R^{n \times n}$,
    $\E\qty[M_i] = 0$,
    and $\norm{M_i} \le C$ almost surely for all $i \in \qty{1, \dots, k}$.
    Then, for all $t \ge 0$, it holds that
    \begin{align}
        \Pr\qty(\norm{\sum_{i=1}^k M_i} \ge t)
        \le 2 n \exp\qty(\frac{-t^2}{2\qty(\sigma^2 + C t/3)}),
    \end{align}
    where $\sigma^2 \doteq \norm{\sum_{i=1}^k \E[M_i^2]}$.
\end{theorem}
\section{Proof of Theoretical Results}
\subsection{Proof of Lemma~\ref{lem:zero-output-update}}
\label{proof:zero-output-update}
\begin{proof}
For any $x$, we have
\begin{align}
    \nabla_v f_\theta(x)=\relu(Wx)
\end{align}
and
\begin{align}
    \nabla_{w_j} f_\theta(x)
    =
    v_j\indicator{w_j^\top x>0}x.
\end{align} 
Therefore,
\begin{align}
    \nabla_v\fL(\theta;\bar\fT(\bar{Y},\ell))
    =
    \frac1n
    \sum_{i=1}^n
    q_i(\theta)\relu(W\bar x_i)
    =
    \frac1n\Phi_\theta(\bar X)^\top q(\theta),
\end{align}
and, for each hidden unit \(j\),
\[
    \nabla_{w_j}\fL(\theta;\bar\fT(\bar{Y},\ell))
    =
    \frac1n
    \sum_{i=1}^n
    q_i(\theta)v_j
    \indicator{w_j^\top\bar x_i>0}\bar x_i.
\]
If \(v=0\), then every first-layer gradient is zero. If additionally
\(\Phi_\theta(\bar X)^\top q(\theta)=0\), then the output-layer gradient is also
zero, so \(\nabla_\theta\fL(\theta;\bar\fT(\bar{\fY},\ell))=0\). Hence gradient descent leaves
\(\theta\) unchanged for all time, which implies zero full-support gain.

For the one-step update formula, when \(\theta_0=(0,W)\), the first-layer gradient is
zero and
\begin{align}
    \begin{cases}
        v_1 &= -\frac{\eta}{n}
        \Phi_{\theta_0}(\bar X)^\top q(\theta_0),\\
        W_1 &= W.
\end{cases}
\end{align}
Thus, we have
\[
    f_{\theta_1}(\bar X)
    =
    \Phi_{\theta_0}(\bar X)v_1
    =
    -\frac{\eta}{n}
    \Phi_{\theta_0}(\bar X)\Phi_{\theta_0}(\bar X)^\top q(\theta_0)
    =
    -\eta G(W)q(\theta_0).
\]
\end{proof}
\subsection{Proof of Lemma~\ref{lem:psd-rank-bookkeeping}}
\label{proof: psd-rank-bookkeeping}
\begin{proof}
Since \(B\) is symmetric positive semidefinite, its singular values are exactly
its eigenvalues. 
Appending zero columns and multiplying by \(\alpha\) gives
nonzero singular values
\[
    \alpha b_1,\dots,\alpha b_r
\]
for \(\Phi_\theta(\bar X)\).
Since \(v=0\), the first-layer portion of
\(\nabla_\theta f_\theta(\bar x_i)\) is zero for every \(i\). 
Hence, the eNTK is
generated only by output-layer gradients, and we have
\[
    \eNTK{\bar X; \theta}
    =
    \Phi_\theta(\bar X)\Phi_\theta(\bar X)^\top
    =
    \alpha^2 B^2.
\]
The nonzero eigenvalues, and thus the nonzero singular values, of this
positive semidefinite matrix are
\[
    \alpha^2 b_1^2,\dots,\alpha^2 b_r^2.
\]
\end{proof}

\subsection{Proof of Lemma~\ref{lem:random-gram-concentration}}
\label{proof: random-gram-concentration}
\begin{proof}
Define \(a \doteq \sqrt{2/n}\) for simplicity. 
For each hidden unit \(j\), define
\(h_j\in\R^n\) by
\[
    (h_j)_i
    \doteq
    \relu((W_{\rand})_{j,i}).
\]
Then \((h_j)_i=a\) with probability \(1/2\) and \(0\) with probability
\(1/2\), independently over \(i\). 
Therefore, it holds that
\[
    G_{\rand}
    =
    \frac1n\sum_{j=1}^M h_jh_j^\top .
\]
For \(i\neq k\), we have
\[
    \E\qty[(h_j)_i(h_j)_k]
    =
    \E\qty[(h_j)_i]\E[(h_j)_k]
    =
    \frac{a^2}{4}
    =
    \frac1{2n},
\]
while
\[
    \E\qty[(h_j)_i^2]
    =
    \frac{a^2}{2}
    =
    \frac1n.
\]
Thus,  we have
\begin{align}
    \E\qty[h_jh_j^\top] = \frac1{2n}(J_n+I_n)
\end{align}
and
\begin{align}
    \bar G
    \doteq
    \E[G_{\rand}]
    =
    \frac{M}{2n^2}(J_n+I_n)
    =
    \lambda(J_n+I_n).
\end{align}

It remains to prove concentration. Let
$A_j\doteq \frac1n h_jh_j^\top$ and
$Q_j\doteq A_j-\E[A_j]$.
Then, \(G_{\rand}-\bar G=\sum_{j=1}^M Q_j\), and the \(Q_j\)'s are independent,
zero-mean, symmetric matrices. 
Since \(\norm{h_j}_2^2\le na^2=2\), it holds that
\[
    \norm{A_j}_2
    =
    \frac1n\norm{h_j}_2^2
    \le
    \frac2n.
\]
By Jensen's inequality, \(\norm{\E[A_j]}_2\le 2/n\), and hence
$\norm{Q_j}_2\le \frac4n$.
Moreover, we have
\[
    \norm{\E\qty[Q_j^2]}_2
    \le
    \norm{\E\qty[A_j^2]}_2
    \le
    \E[\norm{A_j}_2^2]
    \le
    \frac4{n^2}.
\]
Therefore, the matrix-variance parameter satisfies
\[
    \norm{\sum_{j=1}^M\E\qty[Q_j^2]}_2
    \le
    \frac{4M}{n^2}.
\]
By Matrix Bernstein (Theorem~\ref{thm: matrix Bernstein inequality}), there is a universal constant \(C_0>0\), such that with
probability at least \(1-\delta\),
\[
    \norm{G_{\rand}-\bar G}_2
    \le
    C_0
    \qty(
        \frac{\sqrt{M\log(2n/\delta)}}{n}
        +
        \frac{\log(2n/\delta)}{n}).
\]
Let \(z\doteq\log(2n/\delta)\). 
Since \(\lambda=M/(2n^2)\), the last display gives
\[
    \frac{\norm{G_{\rand}-\bar G}_2}{\lambda}
    \le
    2C_0\frac{n\sqrt z}{\sqrt M}
    +
    2C_0\frac{nz}{M}.
\]
If \(M\ge Cn^2\rho z\), then, after increasing the universal constant \(C\) if
necessary, the right-hand side is at most \(1/\sqrt{\rho}\). 
Consequently, we have
\[
    \norm{G_{\rand}-\bar G}_2
    \le
    \frac{\lambda}{\sqrt{\rho}},
\]
as claimed.
\end{proof}

\subsection{Proof of Theorem~\ref{thm:regression-counterexample-condensed}}
\label{proof:regression-counterexample-condensed}
\begin{proof}
Let
\[
    B_{\mathrm{reg}}
    \doteq
    \begin{bmatrix}
        I_{n-2} & 0_{(n-2)\times 2}\\
        0_{2\times(n-2)} & 0_{2\times 2}
    \end{bmatrix}
\]
and $ \alpha\doteq \frac12$.
Choose \(v_{\pre}=0\) and choose \(W_{\pre}\), such that
\[
    \Phi_{\theta_{\pre}}(\bar X)
    =
    \alpha
    \begin{bmatrix}
        B_{\mathrm{reg}} & 0_{n\times(M-n)}
    \end{bmatrix}.
\]
This is realized, for example, by setting
\[
    (W_{\pre})_{j,i}
    =
    \begin{cases}
        \alpha, & j=i\le n-2,\\
        -1, & \text{otherwise}.
    \end{cases}
\]
The nonzero eigenvalues of \(B_{\mathrm{reg}}\) are all equal to \(1\), with
multiplicity \(n-2\). 
Lemma~\ref{lem:psd-rank-bookkeeping} therefore implies
that the nonzero singular values of \(\Phi_{\theta_{\pre}}(\bar X)\) are
\[
    \underbrace{\alpha,\dots,\alpha}_{n-2\text{ times}},
\]
and the nonzero singular values of \(\eNTK{\bar X; \theta_{\pre}}\) are
\[
    \underbrace{\alpha^2,\dots,\alpha^2}_{n-2\text{ times}}.
\]
Thus, we have
\begin{align}
    \begin{cases}
        \enrank_{0.99}\qty(\Phi_{\theta_{\pre}}(\bar X))
        &= \left\lceil 0.99(n-2)\right\rceil;\\
        \enrank_{0.99}\qty(\eNTK{\bar X; \theta_{\pre}})
        &= \left\lceil 0.99(n-2)\right\rceil.
    \end{cases}
\end{align}
The normalized nonzero singular values of \(\Phi_{\theta_{\pre}}(\bar X)\) are
all \(1/(n-2)\), so
\[
    \effrank\qty(\Phi_{\theta_{\pre}}(\bar X))
    = \exp\qty(-\sum_{i=1}^{n-2} \frac1{n-2}\log\frac1{n-2})
    = n-2.
\]
This proves part (i) of Theorem~\ref{thm:regression-counterexample-condensed}.

For the squared loss, at a zero-output checkpoint, we have
\[
    q(\theta_{\pre}) = -\bar Y_{\mathrm{reg}}.
\]
Since \(B_\reg\bar Y_\reg =0\), we have
\[
    \Phi_{\theta_{\pre}}(\bar X)^\top q(\theta_{\pre})
    =
    -\alpha
    \begin{bmatrix}
        B_\reg\\
        0
    \end{bmatrix}
    \bar Y_\reg
    =
    0.
\]
Lemma~\ref{lem:zero-output-update} implies that full-support gradient descent
leaves \(\theta_{\pre}\) unchanged for every learning rate. 
In fact, every individual-example gradient also vanishes. 
For $i\le n-2$, both the prediction and label are zero, so the squared-loss
derivative is zero. 
For $i\in\{n-1,n\}$, all hidden preactivations are
strictly negative and hence the hidden activations and their derivatives
are zero. 
Therefore,
\begin{align}
    \nabla_\theta \ell_{\mathrm{sq}}
    \qty(f_{\theta_{\pre}}(\bar x_i),\bar y_i)=0
    \qquad\text{for every }i=1,\ldots,n.
\end{align}
Every mini-batch gradient at $\theta_{\pre}$ is consequently zero, so
mini-batch SGD stays at $\theta_{\pre}$ for every sequence of batches,
batch size, learning rate, and number of steps.
In addition, it holds that
\[
    \fL(\theta_{\pre};\bar\fT_{\mathrm{reg}})
    =
    \frac1{2n}\norm{\bar Y_{\mathrm{reg}}}_2^2
    =
    \frac1n.
\]
Hence, both \(\fG_{\full}^{(k)}(\theta_{\pre};\bar\fT_{\mathrm{reg}})\)
and the mini-batch gain
\(\fG^{(k)}(\theta_{\pre};\bar\fT_{\mathrm{reg}})\) in
Definition~\ref{def: mini-batch k-step gain} are zero for all
\(k\ge 1\), proving part (ii) of Theorem~\ref{thm:regression-counterexample-condensed}.

It remains to prove the random-initialization claim. 
Let \(\theta_0=\theta_{\rand}\) and \(G_{\rand}\) be as in
Lemma~\ref{lem:random-gram-concentration}, and set
$\lambda=\frac{M}{2n^2}$
and $\eta=\frac1\lambda=\frac{2n^2}{M}$.
By Lemma~\ref{lem:zero-output-update}, since
\(q(\theta_0)=-\bar Y_\reg\), it holds that
\[
    f_{\theta_1}(\bar X)
    =
    \eta G_{\rand}\bar Y_{\mathrm{reg}}.
\]
Furthermore, since \(\bar Y_{\mathrm{reg}}\perp \onevec_n\), we get
\[
    \bar G\bar Y_\reg = \lambda\bar Y_\reg.
\]
Therefore, we have
\[
    f_{\theta_1}(\bar X)-\bar Y_{\mathrm{reg}}
    =
    \eta(G_{\rand}-\bar G)\bar Y_{\mathrm{reg}}.
\]
Applying Lemma~\ref{lem:random-gram-concentration} with
\(\rho=\mgap\), the stated width condition implies that, with probability at
least \(1-\delta\),
\[
    \norm{G_{\rand}-\bar G}_2
    \le
    \frac{\lambda}{\sqrt{\mgap}}.
\]
On this event, it holds that
\[
    \fL\qty(\theta_1;\bar\fT_\reg)
    =
    \frac1{2n}
    \norm{f_{\theta_1}(\bar X)-\bar Y_\reg}_2^2
    \le
    \frac1{2n}
    \frac1{\mgap}
    \norm{\bar Y_\reg}_2^2
    =
    \frac1{\mgap}
    \fL\qty(\theta_0;\bar\fT_\reg).
\]
As a result, we have
\[
    \fG_{\full}^{(1)}(\theta_{\rand};\bar\fT_{\mathrm{reg}})
    =
    \frac{
        \fL(\theta_0;\bar\fT_{\mathrm{reg}})
        -
        \fL(\theta_1;\bar\fT_{\mathrm{reg}})
    }{
        \fL(\theta_0;\bar\fT_{\mathrm{reg}})
    }
    \ge
    1-\frac1{\mgap}.
\]
\end{proof}

\subsection{Proof of Theorem~\ref{thm:classification-counterexample-condensed}}
\label{proof:classification-counterexample-condensed}
\begin{proof}
For simplicity, we write
$\bar Y=\bar Y_{\mathrm{cls}}$ and $\bar S=\bar S_{\mathrm{cls}}= 2\bar Y-\onevec_n$.
By construction, it holds that $\bar S^\top \onevec_n=0$ and $\norm{\bar S}_2^2=n$.
Define $B_{\mathrm{cls}} \doteq I_n+\frac1n\qty(\onevec_n\onevec_n^\top-\bar S\bar S^\top)$
and $\alpha \doteq \frac12$.
Equivalently, we have
\begin{align}
    (B_\cls)_{i,j} =
    \begin{cases}
        1, & i=j,\\
        0, & i\neq j\text{ and }\bar y_i=\bar y_j,\\
        2/n, & \bar y_i\neq \bar y_j.
    \end{cases}
\end{align}
Thus, \(B_\cls\) has nonnegative entries. 
Choose \(v_{\pre}=0\) and choose \(W_{\pre}\), such that
\[
    \Phi_{\theta_{\pre}}(\bar X)
    =
    \alpha
    \begin{bmatrix}
        B_{\mathrm{cls}} & 0_{n\times(M-n)}
    \end{bmatrix}.
\]
One concrete realization is to set, for \(j\le n\),
\[
    (W_{\pre})_{j,i}
    =
    \begin{cases}
        \alpha (B_{\mathrm{cls}})_{i,j},
        & (B_{\mathrm{cls}})_{i,j}>0,\\
        -1,
        & (B_{\mathrm{cls}})_{i,j}=0,
    \end{cases}
\]
and to set \((W_{\pre})_{j,i}=-1\) for all \(j>n\).

The eigenspaces of \(B_{\mathrm{cls}}\) are 
\begin{align}
    \begin{cases}
        B_{\mathrm{cls}}\bar S &= 0;\\
        B_{\mathrm{cls}}\onevec_n &= 2\onevec_n;\\
        B_{\mathrm{cls}}u &= u
    \end{cases}
\end{align}
for every $u\perp\mathrm{span}\qty{\onevec_n,\bar S}$.
Hence, the eigenvalues of \(B_{\mathrm{cls}}\) are
\[
    2,\quad
    \underbrace{1,\dots,1}_{n-2\text{ times}},
    \quad
    0.
\]
By Lemma~\ref{lem:psd-rank-bookkeeping}, the nonzero singular values of
\(\Phi_{\theta_{\pre}}(\bar X)\) are
\[
    2\alpha,
    \quad
    \underbrace{\alpha,\dots,\alpha}_{n-2\text{ times}}.
\]
Therefore, for \(1\le k\le n-1\), the fraction of squared singular-value energy
captured by the top \(k\) singular values is
\[
    \frac{(2\alpha)^2+(k-1)\alpha^2}
         {(2\alpha)^2+(n-2)\alpha^2}
    =
    \frac{k+3}{n+2}.
\]
The smallest \(k\) for which this is at least \(0.99\) is
$\left\lceil 0.99(n+2)-3\right\rceil$.
For the effective rank, the normalized, nonzero singular values are
\[
    \frac2n,
    \quad
    \underbrace{\frac1n,\dots,\frac1n}_{n-2\text{ times}}.
\]
Therefore, we have
\[
    \effrank\qty(\Phi_{\theta_{\pre}}(\bar X))
    =\exp\qty(
        -\frac2n\log\frac2n
        -(n-2)\frac1n\log\frac1n)
    =\frac{n}{2^{2/n}}.
\]
Similarly, the nonzero singular values of \(\eNTK{\bar X; \theta_{\pre}}\) are
\[
    4\alpha^2,
    \quad
    \underbrace{\alpha^2,\dots,\alpha^2}_{n-2\text{ times}}.
\]
Thus, for \(1\le k\le n-1\), the fraction of squared singular-value energy
captured by the top \(k\) eNTK singular values is
\[
    \frac{(4\alpha^2)^2+(k-1)(\alpha^2)^2}
         {(4\alpha^2)^2+(n-2)(\alpha^2)^2}
    =
    \frac{k+15}{n+14},
\]
and hence
\[
    \enrank_{0.99}\qty(\eNTK{\bar X; \theta_{\pre}})
    = \left\lceil 0.99(n+14)-15\right\rceil.
\]
This proves part (i) of Theorem~\ref{thm:classification-counterexample-condensed}.

For the \(0/1\)-label logistic loss,
\[
    \frac{\partial}{\partial z}\ell_{\mathrm{log}}(z,y)
    = \frac{\exp(z)}{1+\exp(z)}-y.
\]
At a zero-output checkpoint, this gives
$q_i(\theta_{\pre}) =\frac12-\bar y_i= -\frac12\bar s_i,$
and $q(\theta_{\pre})=-\frac12\bar S$.
Since \(B_{\mathrm{cls}}\bar S=0\), then we get
\[
    \Phi_{\theta_{\pre}}(\bar X)^\top q(\theta_{\pre})
    =
    -\frac{\alpha}{2}
    \begin{bmatrix}
        B_{\mathrm{cls}}\\
        0
    \end{bmatrix}
    \bar S
    =
    0.
\]
Lemma~\ref{lem:zero-output-update} implies that full-support gradient descent
leaves \(\theta_{\pre}\) unchanged for every learning rate. Since
\(f_{\theta_{\pre}}(\bar x_i)=0\) for all \(i\), we get
\[
    \fL(\theta_{\pre};\bar\fT_{\mathrm{cls}})
    =
    \frac1n\sum_{i=1}^n\log 2
    =
    \log 2.
\]
Consequently, we have
\[
    \fG_{\full}^{(k)}(\theta_{\pre};\bar\fT_{\mathrm{cls}})
    =
    0
\]
for all $k\ge 1$
proving part (ii) of Theorem~\ref{thm:classification-counterexample-condensed}.
Unlike in the regression construction, the individual-example gradients
here are nonzero and cancel in the full-support average
The full-support result alone does not establish zero mini-batch gain.

It remains to prove the random-initialization claim. 
Let \(\theta_0=\theta_{\rand}\) and \(G_{\rand}\) be as in
Lemma~\ref{lem:random-gram-concentration}.
Set $\lambda=\frac{M}{2n^2}$.
At \(\theta_0\), the prediction is zero and
$q(\theta_0)=-\frac12\bar S$.
By Lemma~\ref{lem:zero-output-update}, we have
\[
    f_{\theta_1}(\bar X)
    =
    \frac{\eta}{2}G_{\rand}\bar S.
\]
With
\[
    \eta=\frac{2\gamma_{\mgap}}{\lambda}
    =
    \frac{4\gamma_{\mgap}n^2}{M},
\]
and using
\[
    \bar G\bar S
    =
    \lambda\bar S
\]
by \(\bar S\perp\onevec_n\), we obtain
\[
    f_{\theta_1}(\bar X)-\gamma_{\mgap}\bar S
    =
    \frac{\eta}{2}(G_{\rand}-\bar G)\bar S.
\]
Next, we apply Lemma~\ref{lem:random-gram-concentration} with
\[
    \rho = \qty(\frac{2\gamma_{\mgap}\mgap}{\log 2})^2.
\]
The stated width condition implies \(M\ge Cn^2\rho\log(2n/\delta)\), after
absorbing the numerical factor \((2/\log 2)^2\) into the universal constant
\(C\). 
Hence, with probability at least \(1-\delta\), it holds that
\[
    \frac{\norm{G_{\rand}-\bar G}_2}{\lambda}
    \le
    \frac{\log 2}{2\gamma_{\mgap}\mgap}.
\]
On this event, we get 
\[
    \norm{
        f_{\theta_1}(\bar X)-\gamma_{\mgap}\bar S
    }_2
    \le
    \gamma_{\mgap}
    \frac{\norm{G_{\rand}-\bar G}_2}{\lambda}
    \norm{\bar S}_2.
\]
Since \(\norm{\bar S}_2=\sqrt n\), we have
\[
    \frac1n
    \sum_{i=1}^n
    \abs{
        \bar s_i f_{\theta_1}(\bar x_i)-\gamma_{\mgap}
    }
    =
    \frac1n
    \norm{
        f_{\theta_1}(\bar X)-\gamma_{\mgap}\bar S
    }_1
    \le
    \gamma_{\mgap}
    \frac{\norm{G_{\rand}-\bar G}_2}{\lambda}
    \le
    \frac{\log 2}{2\mgap}.
\]
Let \(\varphi(u) \doteq \log(1+\exp(-u))\). Since \(\abs{\varphi'(u)}\le 1\),
it holds that \(\varphi\) is \(1\)-Lipschitz. Also, for \(y\in\{0,1\}\) and
\(s=2y-1\), we have
\[
    \ell_{\mathrm{log}}(z,y)=\varphi(sz).
\]
Therefore,
\[
    \fL(\theta_1;\bar\fT_{\mathrm{cls}})
    =
    \frac1n\sum_{i=1}^n
    \varphi(\bar s_i f_{\theta_1}(\bar x_i))
    \le
    \varphi(\gamma_{\mgap})
    +
    \frac{\log 2}{2\mgap}.
\]
By the definition
\[
    \gamma_{\mgap}
    =
    \log\!\left(\frac{2\mgap}{\log 2}\right),
\]
we have
\[
    \varphi(\gamma_{\mgap})
    =
    \log(1+\exp(-\gamma_{\mgap}))
    \le
    \exp(-\gamma_{\mgap})
    =
    \frac{\log 2}{2\mgap}.
\]
Thus, we have
\[
    \fL(\theta_1;\bar\fT_{\mathrm{cls}})
    \le
    \frac{\log 2}{\mgap}.
\]
Since \(\fL(\theta_0;\bar\fT_{\mathrm{cls}})=\log 2\), it follows that
\[
    \fG_{\full}^{(1)}(\theta_{\rand};\bar\fT_{\mathrm{cls}})
    =
    \frac{
        \fL(\theta_0;\bar\fT_{\mathrm{cls}})
        -
        \fL(\theta_1;\bar\fT_{\mathrm{cls}})
    }{
        \fL(\theta_0;\bar\fT_{\mathrm{cls}})
    }
    \ge
    1-\frac1{\mgap}.
\]
\end{proof}

\subsection{Proof of Theorem~\ref{thm: optimization readiness one-step gain bound}}
\label{proof: optimization readiness one-step gain bound}
\begin{proof}
    We have by Assumption~\ref{assumption: smoothness}
    \begin{align}
        \fL(\theta_1; \fT)
        &\le \fL(\theta;\fT) - \eta\indot{g(\theta;\fT)}{\emgrad_{\fB_0}(\theta;\fT)} + \frac{\beta\eta^2}{2}\norm{\emgrad_{\fB_0}(\theta;\fT)}^2.
    \end{align}
    Then, by Assumption~\ref{assumption: unbiased estimator} and Theorem 2.1.5 of~\citet{nesterov2013introductory}, we get
    \begin{align}
        \E_{\fB_0}\qty[\fL(\theta_1;\fT)]
        &\le  \fL(\theta;\fT) - \eta\norm{g(\theta;\fT)}^2 + \frac{\beta\eta^2}{2}\E_{\fB_0}\qty[\norm{\emgrad_{\fB_0}(\theta;\fT)}^2]\\
        \fL(\theta; \fT) - \E_{\fB_0}\qty[\fL(\theta_1;\fT)]
        &\ge \eta\norm{g(\theta;\fT)}^2 - 
        \frac{\beta \eta^2}{2}\E_{\fB_0}\qty[\norm{\emgrad_{\fB_0}(\theta;\fT)}^2]\\
        &\ge \eta\norm{g(\theta;\fT)}^2 - 
        \beta \eta^2\E_{\fB_0}\qty[\norm{\emgrad_{\fB_0}(\theta;\fT)}^2]\\
        &=  \eta\norm{g(\theta;\fT)}^2 - 
        \beta\eta^2 \frac{\norm{g(\theta;\fT)}^2}{R(\theta;\fT)}\\
        &=\eta\norm{g(\theta;\fT)}^2\qty(1 - \frac{\beta\eta}{R(\theta;\fT)}).
    \end{align}
    Therefore, we further have
    \begin{align}
        \frac{ \fL(\theta; \fT) - \E_{\fB_0}\qty[\fL(\theta_1; \fT)]}{\fL(\theta; \fT)}
        &\ge \eta\frac{\norm{g(\theta;\fT)}^2}{\fL(\theta; \fT)} \qty(1 - \frac{\beta\eta}{R(\theta;\fT)})\\
        \fG^{(1)}(\theta;\fT) 
        &\ge \eta S(\theta;\fT)\qty(1 - \frac{\beta\eta}{R(\theta;\fT)}).
    \end{align}
    Let $\eta = \alpha R(\theta;\fT)/\beta$.
    Since $0 < \eta < R(\theta;\fT)/\beta$, we then have $0 < \alpha < 1$.
    Plugging in $\eta$, we have
    \begin{align}
        \fG^{(1)}(\theta; \fT)
        &\ge
        \frac{\alpha(1 - \alpha)}{\beta} S(\theta;\fT)R(\theta;\fT)\\
        &= \frac{\alpha(1 - \alpha)}{\beta}\optread(\theta;\fT).
    \end{align}
\end{proof}

\section{Additional Experiment Details}
\subsection{Estimation Details}
\label{app:estimation-details}

\compactpara{$k$-Step Gain Estimation.}
Let $\fT_\val \doteq (\fX,\fY,P_{\fX\fY}^\val)$ be a validation task and
let $\theta$ be the checkpoint whose trainability we wish to estimate.
We denote the validation set by
\[
    \fD_\val \doteq \qty{(x_i,y_i)}_{i=1}^{N_\val},
    \qquad (x_i,y_i) \sim P_{\fX\fY}^\val .
\]
We estimate the population loss $\fL(\theta;\fT_\val)$ using the empirical
validation loss
\[
    \emloss_{\fD_\val}(\theta;\fT_\val)
    \doteq
    \frac{1}{N_\val}
    \sum_{(x_i,y_i)\in\fD_\val}
    \ell(f_\theta(x_i),y_i).
\]
In all experiments, we use $N_\val=10{,}000$.

Starting from $\theta_0=\theta$, we run stochastic gradient descent for
$k$ steps,
\[
    \theta_{s+1}
    =
    \theta_s
    -
    \eta \emgrad_{\fB_s}(\theta_s;\fT_\val),
    \qquad s=0,\dots,k-1,
\]
where each mini-batch $\fB_s$ consists of $m=4$ samples drawn uniformly
with replacement from $\fD_\val$.
We use $\eta=10^{-3}$ for the gain-estimation rollouts.
For each checkpoint and validation task, we repeat this procedure for
$\fR=128$ independent rollouts and obtain terminal parameters
$\qty{\theta_k^{(i)}}_{i=1}^{\fR}$.
The expected post-update empirical loss is estimated as
\[
    \emloss_\fR(\theta_k;\fT_\val)
    \doteq
    \frac{1}{\fR}
    \sum_{i=1}^{\fR}
    \emloss_{\fD_\val}\qty(\theta_k^{(i)};\fT_\val).
\]
We then estimate the $k$-step gain by
\[
    \widehat{\fG}^{(k)}(\theta;\fT_\val)
    \doteq
    \frac{
        \emloss_{\fD_\val}(\theta;\fT_\val)
        -
        \emloss_\fR(\theta_k;\fT_\val)
    }{
        \emloss_{\fD_\val}(\theta;\fT_\val)
    }.
\]

\paragraph{Optimization Readiness Estimation.}
We estimate the population gradient $g(\theta;\fT_\val)$ using the
gradient of the full validation set
\[
    \widehat g(\theta;\fT_\val)
    \doteq
    \nabla_\theta \emloss_{\fD_\val}(\theta;\fT_\val).
\]
The empirical gradient strength is
\[
    \widehat S(\theta;\fT_\val)
    \doteq
    \frac{
        \norm{\widehat g(\theta;\fT_\val)}_2^2
    }{
        \emloss_{\fD_\val}(\theta;\fT_\val)
    }.
\]
To estimate gradient reliability, we independently sample
$\fR=128$ mini-batches $\qty{\fB_i}_{i=1}^{\fR}$ of size $m=4$ from
$\fD_\val$ with replacement and compute
\[
    \widehat R(\theta;\fT_\val)
    \doteq
    \frac{
        \norm{\widehat g(\theta;\fT_\val)}_2^2
    }{
        \frac{1}{\fR}
        \sum_{i=1}^{\fR}
        \norm{\emgrad_{\fB_i}(\theta;\fT_\val)}_2^2
    }.
\]
The estimated optimization readiness is then
\[
    \widehat{\optread}(\theta;\fT_\val)
    \doteq
    \widehat S(\theta;\fT_\val)\widehat R(\theta;\fT_\val).
\]

\paragraph{Baseline Metric Estimation.}
Let $X_\val$ denote the matrix of validation inputs, where the $i$-th row
is $x_i^\top$.
For both SCR and P-MNIST, we compute the representation effective rank
and 99\% energy rank as
\[
    \effrank(\Phi_\theta(X_\val)),
    \qquad
    \enrank_{0.99}(\Phi_\theta(X_\val)).
\]
We also compute the active-neuron fraction on the full validation set
using threshold $\tau_\act=0.1$.

For SCR, the network has a manageable number of parameters, so we compute
the eNTK rank and exact Hessian rank directly:
\[
    \enrank_{0.99}\qty(\eNTK{X_\val;\theta}),
    \qquad
    \enrank_{0.99}
    \qty(
        \nabla_\theta^2
        \emloss_{\fD_\val}(\theta;\fT_\val)
    ).
\]
For P-MNIST, we omit the eNTK rank because $f_\theta(x)$ is vector-valued,
so $\nabla_\theta f_\theta(x)$ is a Jacobian rather than a vector and the
resulting empirical kernel is not directly comparable to the scalar-output
case.

For the Hessian-rank metric in P-MNIST, we follow the Gram-matrix
approximation used by~\citet{lewandowski2024direction}.
Specifically, we sample $b=128$ input-label pairs from $\fD_\val$ and form
\[
    G
    \doteq
    \mqty[
        \nabla_\theta \ell(f_\theta(x_1),y_1)
        &
        \cdots
        &
        \nabla_\theta \ell(f_\theta(x_b),y_b)
    ].
\]
We then use
\[
    \enrank_{0.99}(G^\top G)
\]
as the approximated 99\% energy rank of the Hessian.

\subsection{Compute Resources}
Experiments were run on a shared Slurm cluster.
Since jobs were scheduled opportunistically, the exact hardware allocation
varied across runs.
SCR experiments were faster on CPU and took less than 3 hours per seed.
P-MNIST experiments benefited from GPU training and took approximately
10.5 hours per seed on an NVIDIA RTX 4000 GPU with 8GB of GPU memory.
\subsection{Figures}

\begin{figure}[htbp]
    \centering
    \begin{subfigure}{0.48\textwidth}
        \centering
        \includegraphics[width=\linewidth]{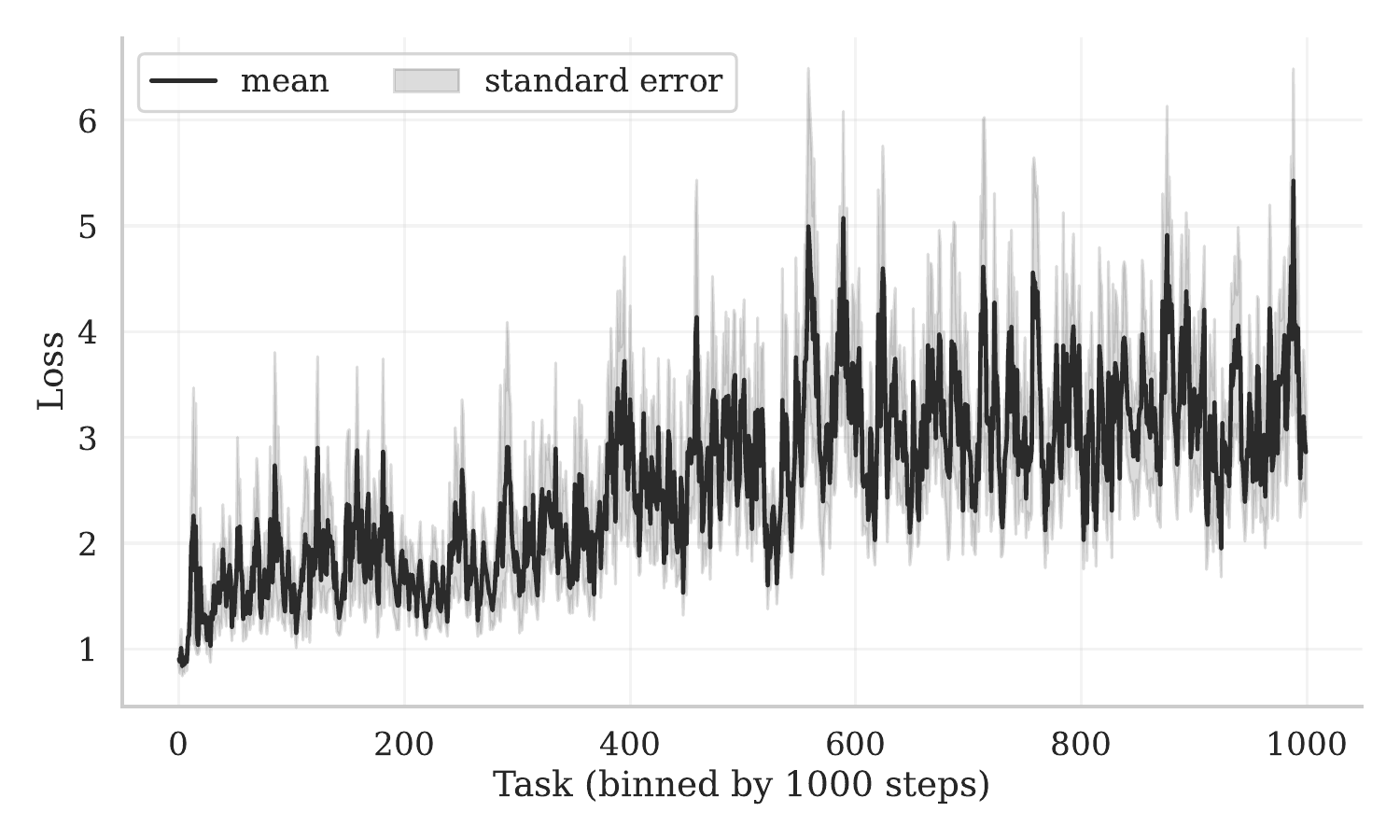}
        \caption{Slowly-Changing Regression}
        \label{fig:slow_reg}
    \end{subfigure}
    \hfill
    \begin{subfigure}{0.48\textwidth}
        \centering
        \includegraphics[width=\linewidth]{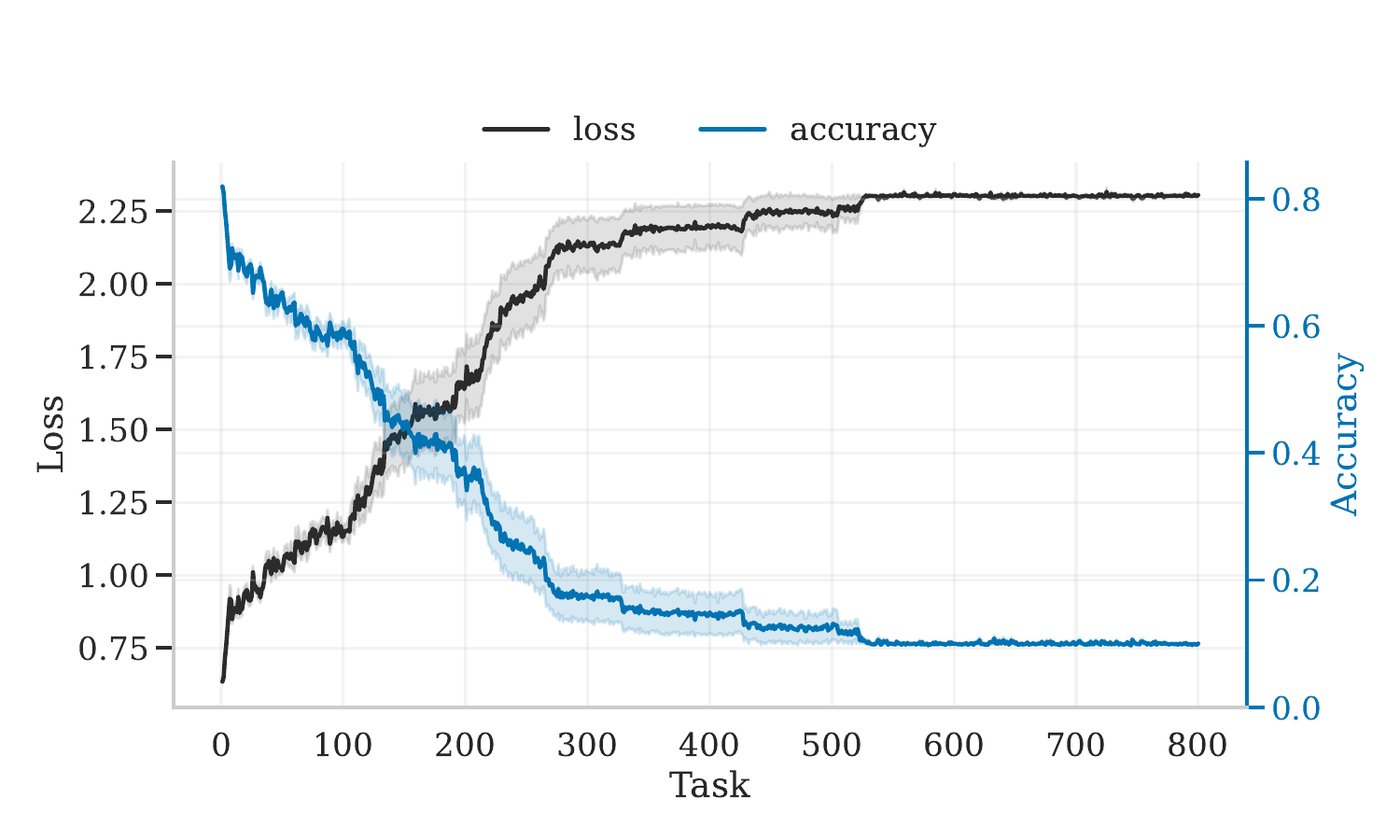}
        \caption{Permuted MNIST}
        \label{fig:perm_mnist}
    \end{subfigure}
    \caption{Learning curves of Slowly-Changing Regression and Permuted MNIST.}
    \label{fig:learning curve}
\end{figure}

\begin{figure}[htbp]
    \centering
    \begin{subfigure}{\textwidth}
        \centering
        \includegraphics[width=\linewidth]{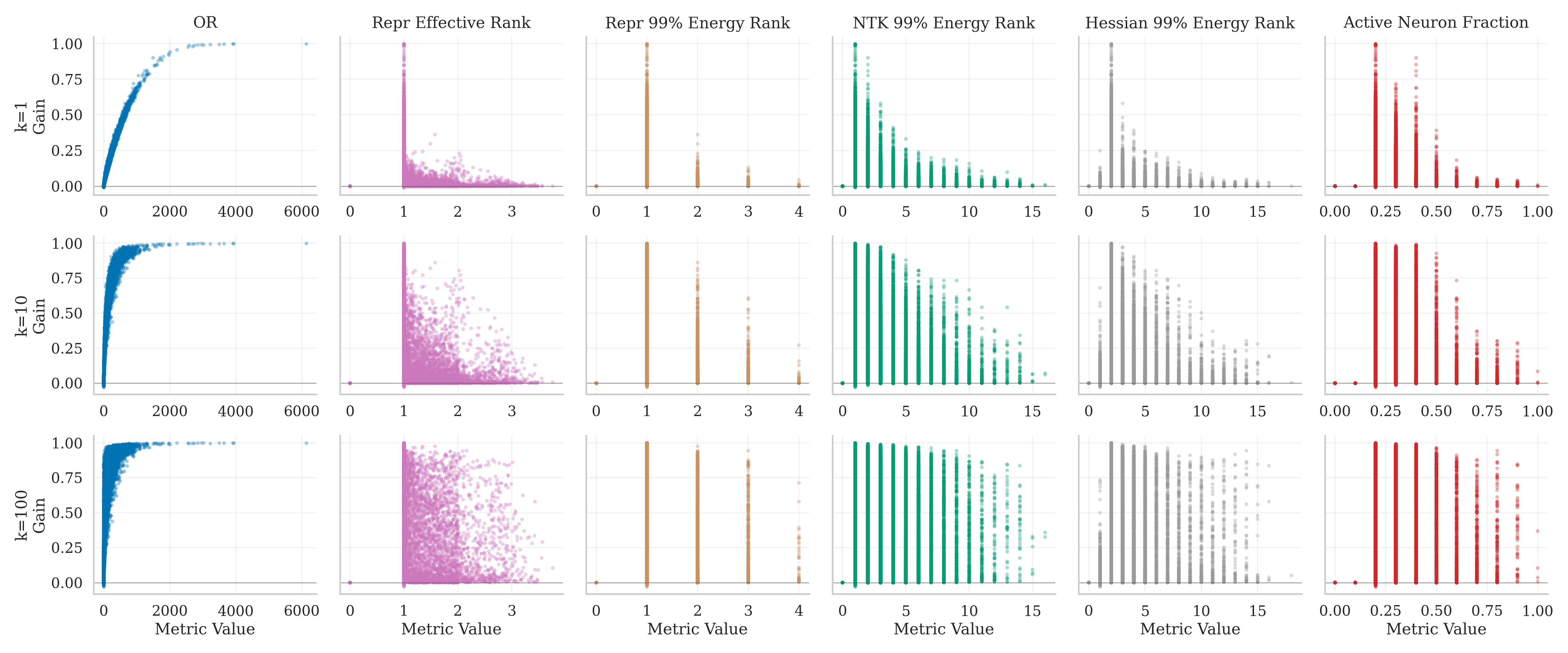}
        \caption{Slowly-Changing Regression. Results aggregated from 20 training runs and 30 validation tasks.}
        \label{fig:slow reg scatter}
    \end{subfigure}
    \vspace{0.5cm} 
    \begin{subfigure}{\textwidth}
        \centering
        \includegraphics[width=\linewidth]{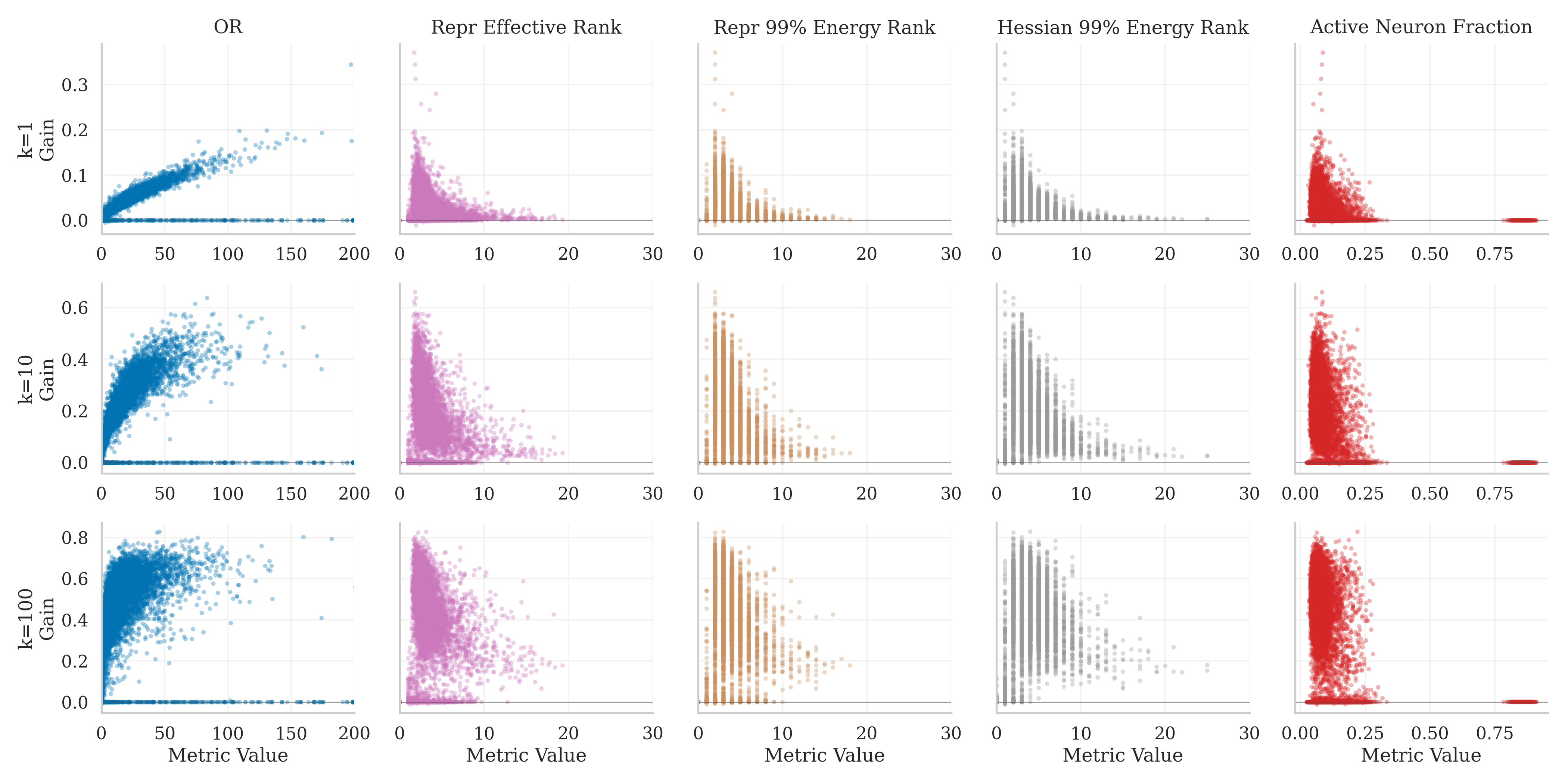}
        \caption{Permuted MNIST. Results aggregated from 20 training runs and 10 validation tasks.}
        \label{fig:perm_mnist scatter}
    \end{subfigure}
    \caption{1-, 10-, and 100-step gain vs. plasticity metric values for Slowly-Changing Regression and Permuted MNIST (extreme metric values clipped to account for outliers).}
    \label{fig:scatter plot}
\end{figure}

\begin{figure}[htbp]
    \centering
    \begin{subfigure}{\textwidth}
        \centering
        \includegraphics[width=\linewidth]{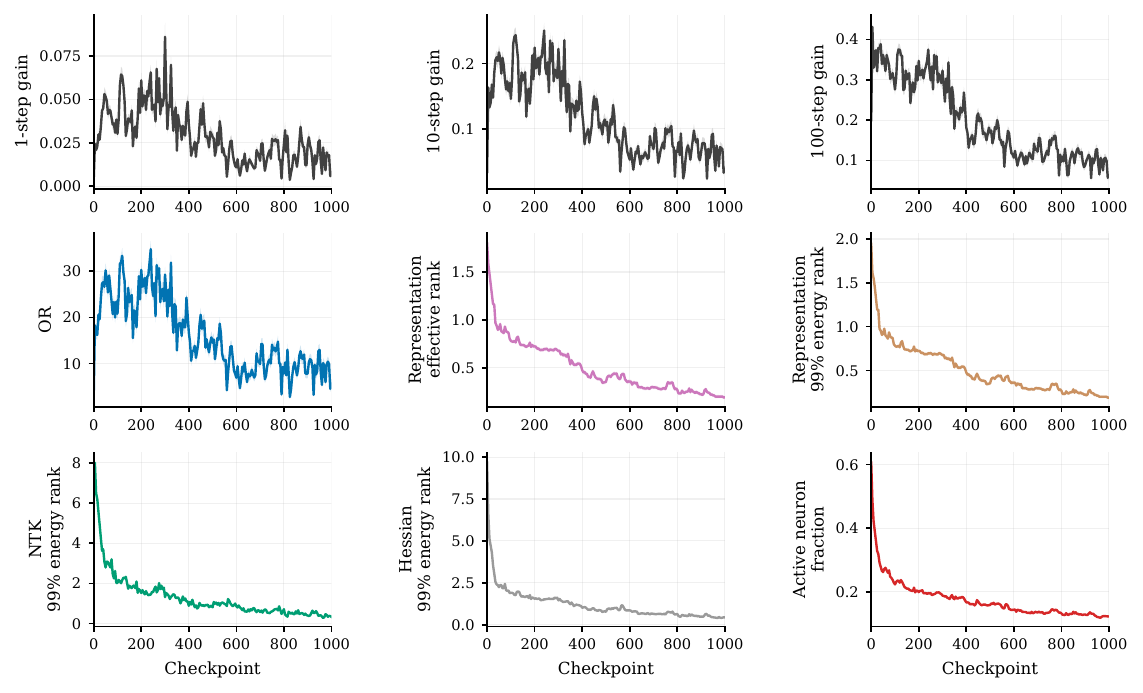}
        \caption{Slowly-Changing Regression. Results aggregated from 20 training runs and 30 validation tasks.}
        \label{fig:slow reg trajectory}
    \end{subfigure}
    \vspace{0.5cm} 
    \begin{subfigure}{\textwidth}
        \centering
        \includegraphics[width=\linewidth]{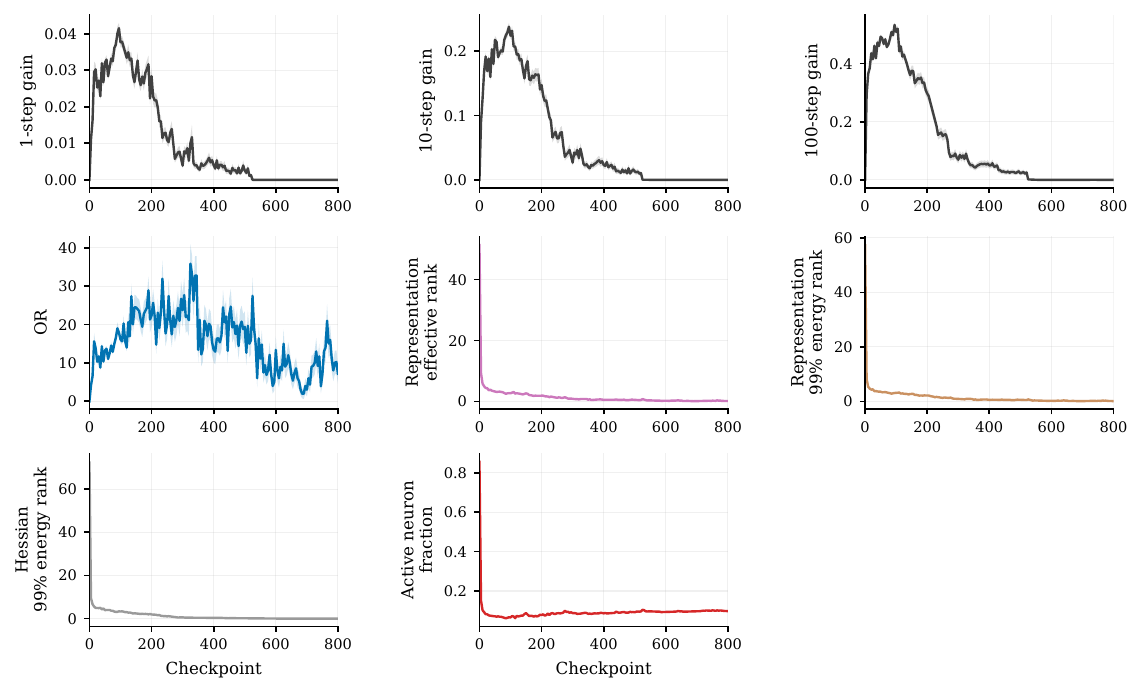}
        \caption{Permuted MNIST. Results aggregated from 20 training runs and 10 validation tasks.}
        \label{fig:perm_mnist trajectory}
    \end{subfigure}
    \caption{1-, 10-, and 100-step gain and plasticity metric values against checkpoints for Slowly-Changing Regression and Permuted MNIST (extreme metric values clipped to account for outliers). Shaded areas denote standard errors.}
    \label{fig:trajectory plot}
\end{figure}

\begin{figure}[htbp]
    \centering
    \begin{subfigure}{\textwidth}
        \centering
        \includegraphics[width=\linewidth]{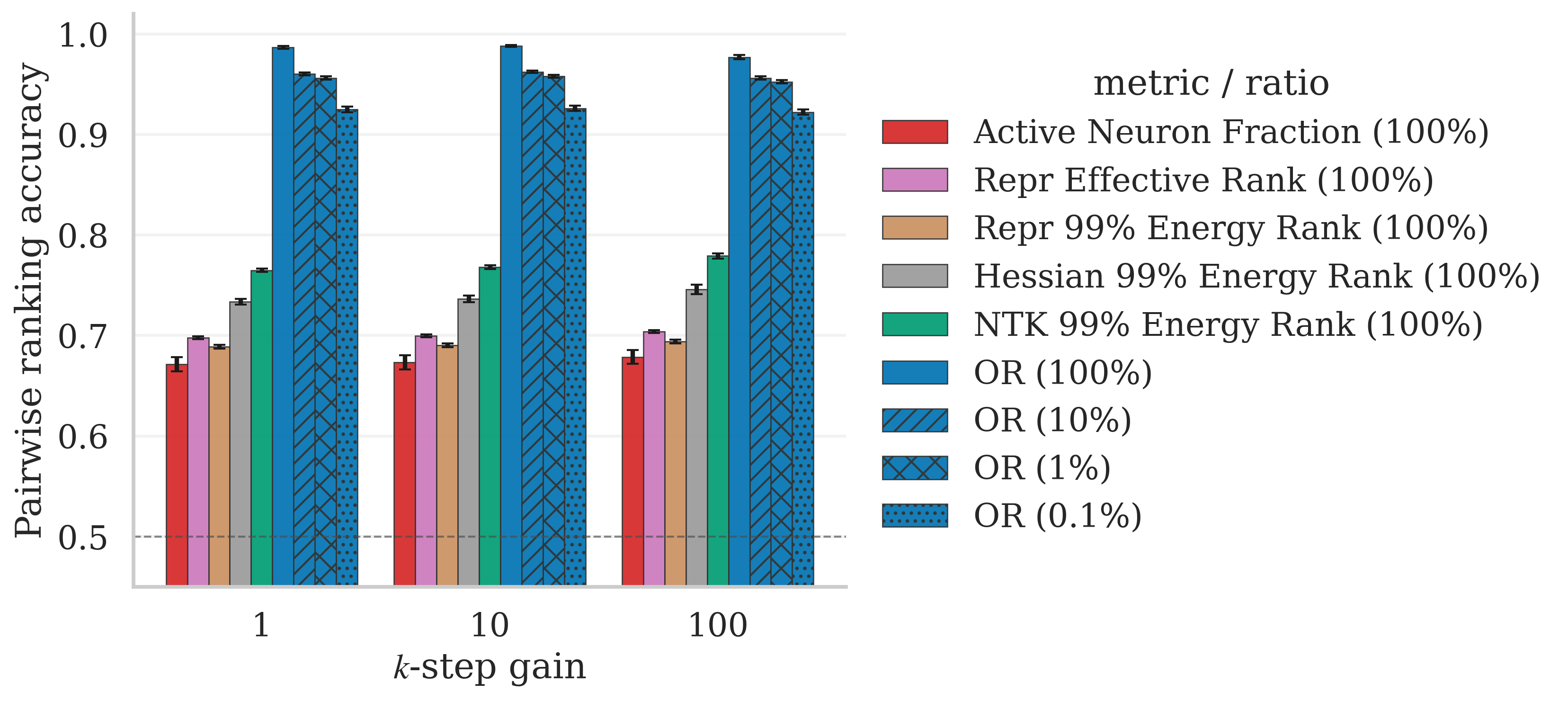}
        \caption{Slowly-Changing Regression. Results aggregated from 20 training runs and 30 validation tasks.}
        \label{fig:slow reg subsample}
    \end{subfigure}
    \vspace{0.5cm} 
    \begin{subfigure}{\textwidth}
        \centering
        \includegraphics[width=\linewidth]{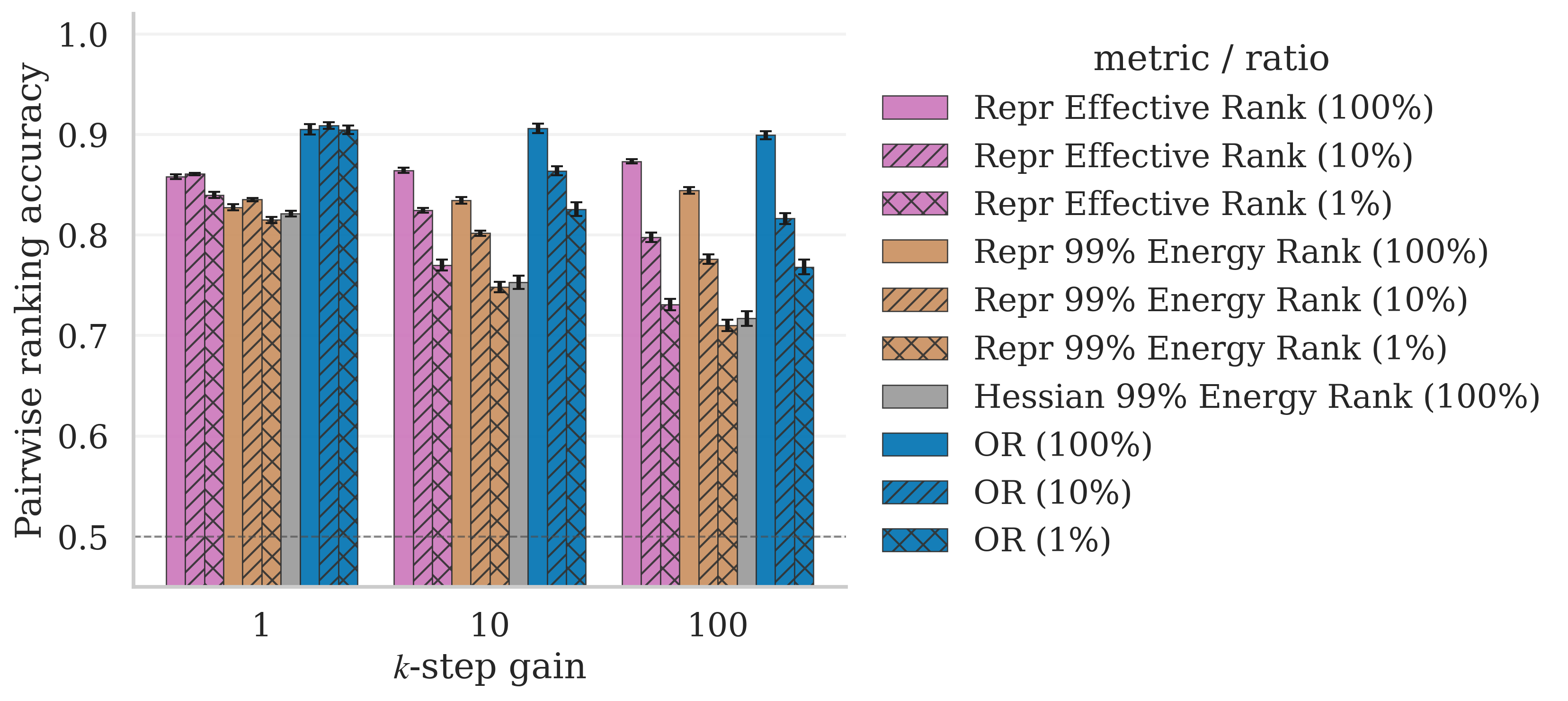}
        \caption{Permuted MNIST. Results aggregated from 20 training runs and 10 validation tasks. We omit the active neuron fraction because its accuracy is all under 50\%.}
        \label{fig:perm_mnist subsample}
    \end{subfigure}
    \caption{Subsampling ablation study results for Slowly-Changing Regression and Permuted MNIST. The error bars denote standard errors.}
    \label{fig:subsample bar plot}
\end{figure}

\clearpage

\end{document}